\documentclass[runningheads]{llncs}

\usepackage{epsfig}
\usepackage{amssymb}
\usepackage[pagebackref=true,breaklinks=true,letterpaper=true,colorlinks,bookmarks=false]{hyperref}

\usepackage[utf8]{inputenc} 
\usepackage[T1]{fontenc}    
\usepackage{url}            
\usepackage{booktabs}       
\usepackage{amsfonts}       
\usepackage{nicefrac}       
\usepackage{microtype}      

\usepackage{mathtools}
\usepackage{amssymb}
\usepackage{graphicx}
\graphicspath{{./figures/}}
\usepackage{subcaption}
\usepackage{enumitem}
\usepackage{algorithm}
\usepackage{algorithmic}

\newcommand{\argmax}{\mathop{\arg\max}}

\newcommand{\ep}{\mathbb{E}}

\newcommand{\data}{\mathcal{D}}

\newcommand{\RR}{\mathbb{R}}

\newcommand{\fig}[1]{Fig.~\ref{fig:#1}}
\newcommand{\tabl}[1]{Table~\ref{table:#1}}
\newcommand{\eqn}[1]{Eqn.~\eqref{eqn:#1}}
\newcommand{\secref}[1]{Sec.~\ref{sec:#1}}

\newcommand\blfootnote[1]{%
  \begingroup
  \renewcommand\thefootnote{}\footnote{#1}%
  \addtocounter{footnote}{-1}%
  \endgroup
}

\begin{document}
\title{Learning Implicit Generative Models By Teaching Density Estimators}
\author{Kun Xu$^{\dagger}$ \and
Chao Du$^{\dagger}$ \and
Chongxuan Li \and 
Jun Zhu$^{\ddag}$ \and
Bo Zhang}
\institute{Dept. of Comp. Sci. \& Tech., Institute for AI, BNRist Center,\\
  Tsinghua-Bosch ML Center, THBI Lab, Tsinghua University, Beijing, China
\email{ \{kunxu.thu,duchao0726,chongxuanli1991\}@gmail.com, \{dcszj,~dcszb\}@tsinghua.edu.cn}\blfootnote{$^\dagger$ Equal contribution. $^\ddag$ Corresponding author.}}

\maketitle
\begin{abstract}
Implicit generative models are difficult to train as no explicit density functions are defined. Generative adversarial nets (GANs) present a minimax framework to train such models, which however can suffer from mode collapse due to the nature of the JS-divergence. This paper presents a {\it learning by teaching} (LBT) approach to learning implicit models, which intrinsically avoids the mode collapse problem by optimizing a KL-divergence rather than the JS-divergence in GANs. In LBT, an auxiliary density estimator is introduced to fit the implicit model's distribution while the implicit model teaches the density estimator to match the data distribution. LBT is formulated as a bilevel optimization problem, whose optimal generator matches the true data distribution. LBT can be naturally integrated with GANs to derive a hybrid LBT-GAN that enjoys complimentary benefits. Finally, we present a stochastic gradient ascent algorithm with unrolling to solve the challenging learning problems. Experimental results demonstrate the effectiveness of our method.
\end{abstract}

\vspace{-.1cm}
\section{Introduction} \label{sec:introduction}
\vspace{-.1cm}

Deep generative models (DGMs)~\cite{kingma2013auto,goodfellow2014generative,oord2016pixel} are powerful tools to capture the distributions over complicated manifolds (e.g., natural images), especially the 
recent developments of implicit statistical models~\cite{radford2015unsupervised,arjovsky2017wasserstein,karras2017progressive},
also called implicit probability distributions~\cite{mohamed2016learning}.
Implicit models are flexible by adopting a sampling procedure rather than a tractable density. However, they are difficult to learn, partly because maximum likelihood estimation (MLE) is not directly applicable.

Generative adversarial networks (GANs)~\cite{goodfellow2014generative} address this difficulty by adopting a minimax game,
where a discriminator $D$ is introduced 
to distinguish whether a sample is real (i.e., from the data distribution) or fake (i.e., from a generator $G$), while $G$ tries to fool $D$ via generating realistic samples.
Although GANs can produce high quality samples, it suffers from lacking sample diversity, also known as the mode collapse problem~\cite{goodfellow2016nips},
which still remains unaddressed.

A compelling reason
for mode collapse arises from the objective function optimized by GANs~\cite{nguyen2017dual}, which is shown to minimize the JS-divergence between the data distribution $p_{\data}$ and the generator distribution $p_G$~\cite{goodfellow2014generative}. 
As shown in~\cite{huszar2015not,theis2015note} and illustrated in Fig.~\ref{fig:jsd_kl}, JS-divergence can be tolerant to mode collapse whereas the $KL(p_{\data}||p_G)$ achieves its optima iff $p_{\data}=p_G$.
\cite{nowozin2016f} enable us to train implicit models via KL-divergence using importance sampling, i.e., estimating the KL-divergence using generated samples. However, it may also fail in practice~\cite{metz2016unrolled} as the KL-divergence will be under-estimated if the generated samples do not capture all modes in training data.

To address the above issues, we propose {\it learning by teaching} (LBT), a novel framework to learn implicit models. LBT can be shown to optimizes the KL-divergence, which is more resistant to mode collapse than the JS-divergence due to the zero-avoiding properties~\cite{nasrabadi2007pattern}.  
In LBT, we {\it learn} an implicit generator $G$
by {\it teaching} a density estimator $E$ 
to match the data distribution. 
The training scheme consists of two parts:
\begin{enumerate}[label=(\alph*),leftmargin=*]
    \item The estimator $E$ is trained to maximize the log-likelihood of the samples of the generator $G$;
    \label{step1}
    \item The generator $G$'s goal is to improve the performance of the trained estimator in terms of the log-likelihood of real data samples.
    \label{step2}
\end{enumerate}

\begin{figure}
\centering
    \begin{subfigure}[t]{0.41\textwidth}
        \includegraphics[width=\textwidth]{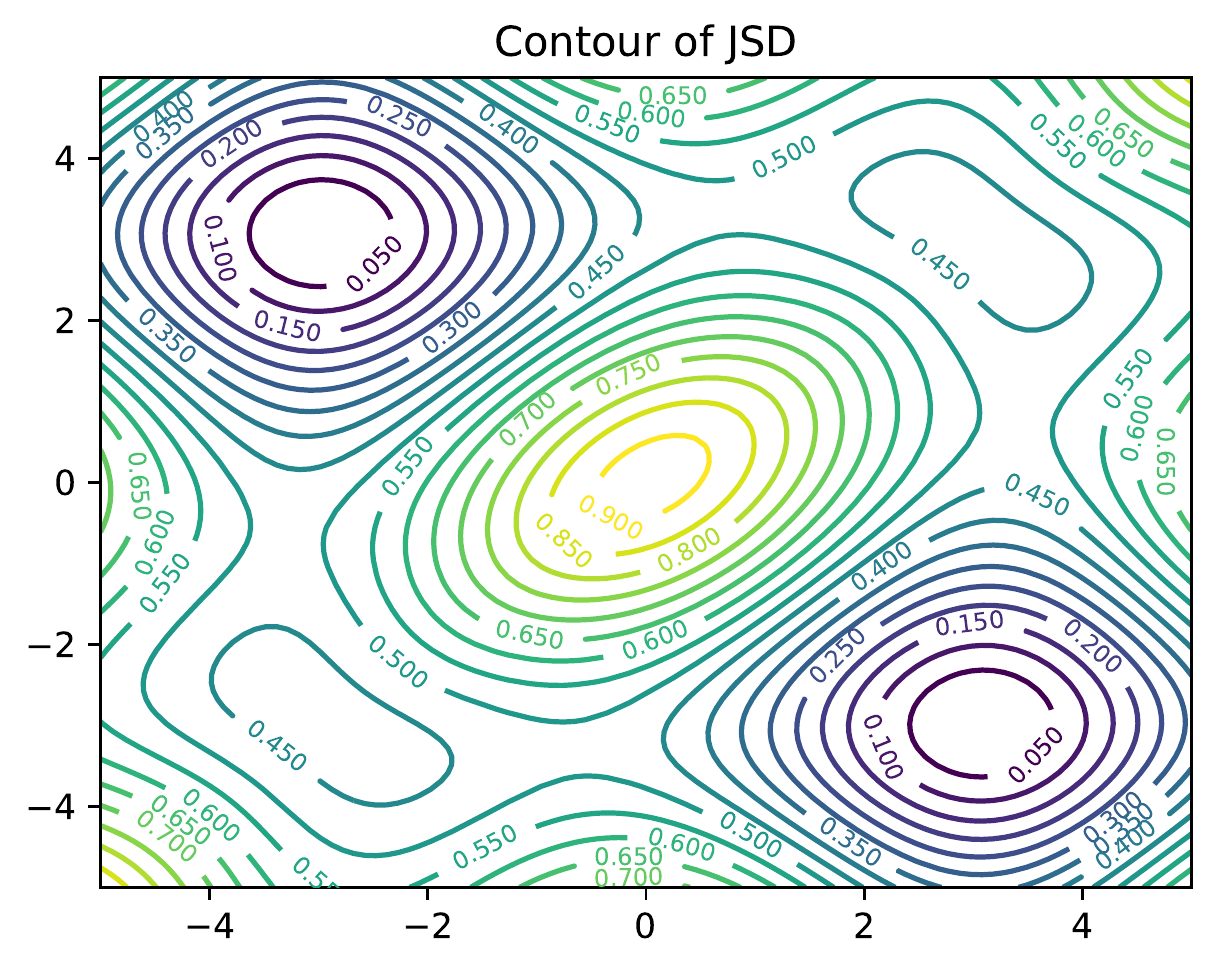}
    \end{subfigure}
    \begin{subfigure}[t]{0.41\textwidth}
        \includegraphics[width=\textwidth]{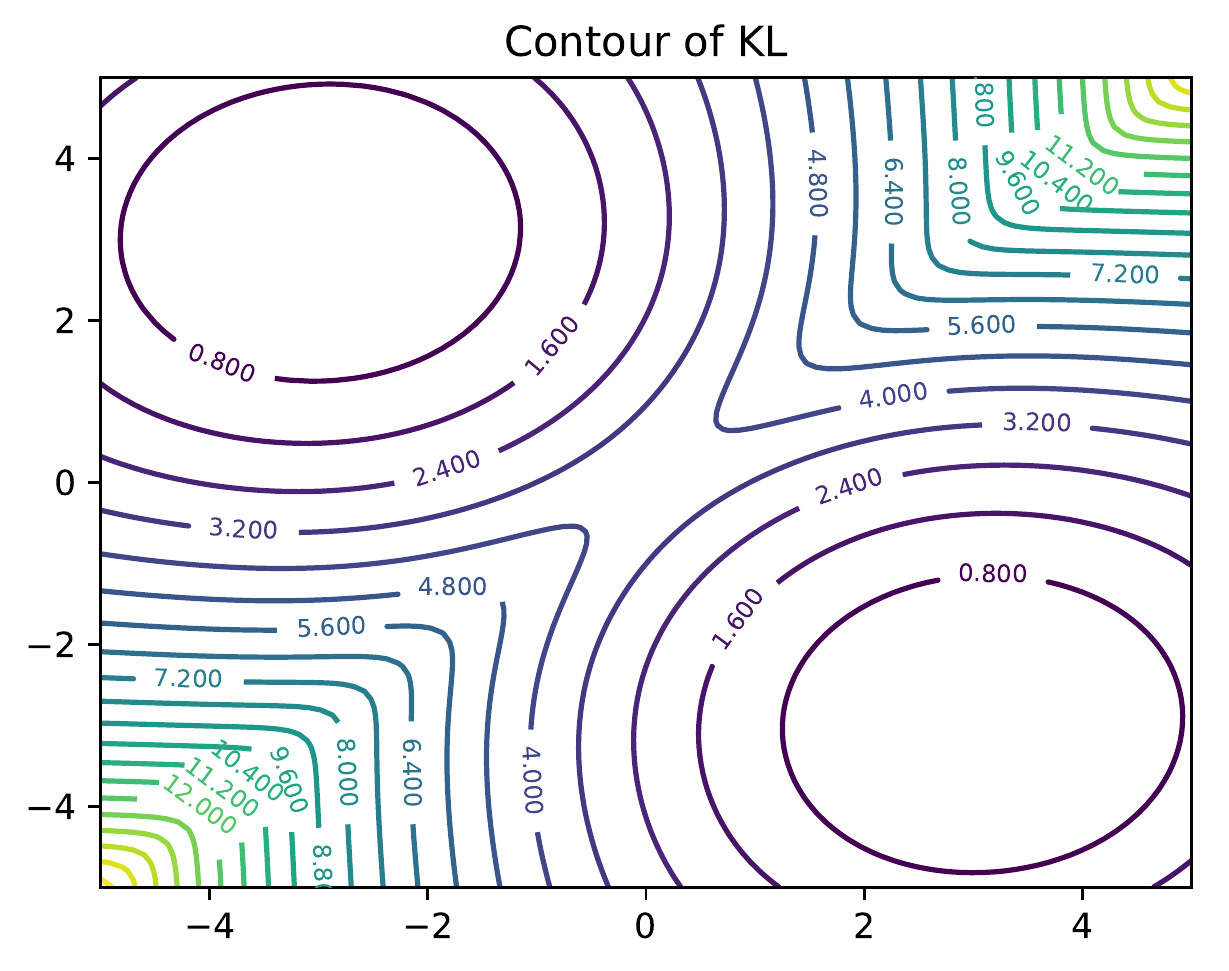}
    \end{subfigure}
	\caption{Suppose the data distribution is a mixture of Gaussian (MoG), i.e., $p_{\data}(x) = 0.5 \mathcal{N}(-3, 1) + 0.5 \mathcal{N}(3, 1)$, and the model distribution is a MoG with learnable means $\theta=[\theta_1, \theta_2]$, i.e., $p_G(x) = 0.5\mathcal{N}(\theta_1, 1) + 0.5\mathcal{N}(\theta_2, 1)$. The figure shows the contours of the two divergences with the $x$-axis and $y$-axis representing the value of $\theta_1$ and $\theta_2$ respectively. The JS-divergence allows mode-collapsed local optima while the KL-divergence does not.}
	\label{fig:jsd_kl}
    \vspace{-.2cm}
\end{figure}

Though in both LBT and GAN, an auxiliary model is introduced to help the training of the generator,
the role of $E$ in LBT is significantly different from that of $D$ in GAN, and they are complimentary to each other. 
The estimator $E$ in LBT penalize $G$ for missing modes in training data, whereas the discriminator $D$ in GAN penalize $G$ for generating unrealistic samples.
In LBT, $E$ always tracks $p_G$ and 
once $p_G$ misses some modes, the estimator $E$ will also miss them, resulting a poor likelihood of real data samples, which penalize $G$ heavily.
In such a manner, the estimator in LBT directs the generated samples to overspread the support of data distribution. 
In contrast, the goal of $D$ in the vanilla GAN is to distinguish whether a sample is real or fake.
Therefore, during the competing with $D$, $G$ will be penalized much more heavily for generating unrealistic samples than missing modes. Based on this insight, we further 
conjoin the complimentary advantages to develop LBT-GAN,
which augments LBT with a discriminator network.
In LBT-GAN, 
$E$ helps $G$ to overspread the data distribution and $D$ helps $G$ to generate realistic samples. 


Formally, LBT (and LBT-GAN) is formulated as a bilevel optimization~\cite{colson2007overview} problem, where an {\it upper} level optimization problem (i.e., part~\ref{step2}) depends on the optimal solution of a {\it lower} level problem (i.e., part~\ref{step1}).
The gradients of the upper problem w.r.t.\ the parameters of $G$ are intractable since the optimal solution of $E$ cannot be analytically expressed by $G$'s parameters. 
We propose to use the unrolling technique~\cite{metz2016unrolled} to efficiently approximate the gradients.
Under nonparametric conditions, the optimum of LBT (and LBT-GAN) is achieved when both the generator and the estimator converge to the data distribution. 
Besides, we further analyze that an estimator with insufficient capability can still help $G$ to resist to mode collapse in LBT-GAN.
Experimental results on both synthetic and real datasets demonstrate the effectiveness of LBT and LBT-GAN.


\vspace{-.1cm}
\section{Background} \label{sec:background}
\vspace{-.1cm}



Consider an implicit generative model $G(\cdot;\theta)$ parameterized by $\theta$ that maps a simple random variable $z\in\RR^H$ to a sample $x$ in the data space $\RR^L$, i.e., $x=G(z;\theta)$, where $H$ and $L$ are the dimensions of the random variables and the data samples, respectively.
Typically, $z$ is drawn from a standard Gaussian distribution $p_Z$ and $G$ is a feed-forward neural network.
The sampling procedure defines a distribution $p_G(x;\theta)$ over the data space. 
The goal of the generator $G$ is to approximate the data distribution $p_{\data}(x)$, i.e., to produce samples of high quality and diversity.

Since the generator distribution is implicit, it is infeasible to adopt MLE directly to train the generator.
To address this problem, GANs~\cite{goodfellow2014generative} adopt a minimax game, where a discriminator $D(\cdot;\psi)$ parameterized by $\psi$ is introduced to distinguish generated samples from true data samples,
while the generator $G$ tries to fool $D$ via generating realistic samples. The parameters of $G$ and $D$ are learned by solving a minimax game:
\begin{align}
\min_{\theta}\max_{\psi} f_{\textrm{GAN}}(\theta,\psi)\vcentcolon= \ep_{x\sim p_{\data}}[\log D(x;\psi)] \nonumber \\
+ \ep_{z\sim p_Z} [\log(1-D(G(z;\theta);\psi))].
\label{eqn:obj-gan}
\end{align}
\cite{goodfellow2014generative} show that the discriminator achieves its optimum when $D(x) = \frac{p_{\data}(x)}{p_{\data}(x) + p_G(x)}$, and solving the minimax problem is equivalent to minimizing the JS-divergence between $p_{\data}(x)$ and $p_G(x)$,
whose optimal point is $p_G=p_{\data}$, under the assumption that $G$ and $D$ have infinite capacity.
However, GANs can suffer from the mode collapse problem for both theoretical reasons~\cite{nguyen2017dual,huszar2015not} and practical reasons ~\cite{metz2016unrolled,srivastava2017veegan,arjovsky2017wasserstein}.

From the theoretical perspective, previous work has investigated the mode collapse nature of JS-divergence~\cite{nguyen2017dual}.
By optimizing the JS-divergence, the generative model tends to cover certain modes, rather than overspreading the data distribution~\cite{theis2015note}, thus leading to mode collapse in GANs.
\fig{jsd_kl} (left) presents a simple example, 
where the local optima with mode collapse can still be found by optimizing the JS-divergence, even if $p_G$ is flexible enough. In contrast, the KL-divergence can overcome this problem because of the zero-avoiding property~\cite{nasrabadi2007pattern}, and \fig{jsd_kl} (right) shows that $KL(p_{\data}||p_G)$ achieves its optima iff $p_G=p_{\data}$.


There are previous attempts on training implicit models by optimizing other divergence, including the KL-divergence~\cite{nowozin2016f,nguyen2017dual}.
For instance, D2GAN~\cite{nguyen2017dual} uses an auxiliary discriminator to diversify the generator distribution, which introduces the KL-divergence into the objective function.
However, it practically fails as the discriminators in D2GAN are fixed during the update of the generator, 
which makes that the gradient of the KL-divergence w.r.t.\ the generator cannot be propagated through the discriminator and breaks the zero-avoiding property of the KL-divergence.
{nowozin2016f} propose to estimate the KL-divergence using importance sampling, i.e., $KL(p_{\data}||p_G) = \ep_{p_G}[ \frac{p_{\data}}{p_G}\log \frac{p_{\data}}{p_G}]$. 
However, the estimation will be of large variance if the generator fails to capture all modes in data as it is difficult to draw a sample in the missed modes in $p_{\data}$ and the KL-divergence tends to be under-estimated. Therefore, once the generator distribution collapsed, the estimated KL-divergence cannot penalize the generator for missing modes and
encourage the generator to capture all modes in training data.


\vspace{-.1cm}
\section{Method}
\vspace{-.1cm}

To address the mode collapse issue, we present a novel framework \textit{learning by teaching} (LBT), which enables us to learn implicit models by optimizing the KL-divergence between $p_{\data}(x)$ and $p_G(x)$. 


\vspace{-.1cm}
\subsection{Learning by Teaching (LBT)}\label{sec:method:lbt}
\vspace{-.1cm}


We introduce an auxiliary density estimator $E$ with density $p_E(x;\phi)$ parameterized by $\phi$ to learn the distribution defined by the implicit generator $G(\cdot;\theta)$. The estimator $E$ provides a surrogate density for $G$ to estimate the KL-divergence between $p_{\data}$ and $p_G$.
Specifically, in LBT, the estimator $E$'s goal is to learn $p_G$ via MLE,
i.e., by maximizing the likelihood evaluated on samples generated from $G$.
And the generator's goal is to maximize $E$'s likelihood evaluated on real data samples,
which is possible since the generator's samples decide the training process of $E$.
As a consequence, $E$ only captures the modes of its ``training data'', i.e., the generated samples, and has low density for unseen data.
To avoid the penalty from the real samples, the generator $G$ has to overspread the true data distribution and cover all modes.
Formally, LBT is defined as a bilevel optimization problem~\cite{colson2007overview}:
\begin{align}\label{eqn:obj}
\max_\theta& \quad \ep_{x\sim p_{\data}(x)}[\log p_E(x;\phi^\star(\theta))] \\
\textrm{s.t.}&  \quad \phi^\star(\theta) = \argmax_{\phi} \ep_{z\sim p_Z} [\log p_E(G(z;\theta);\phi)],
\end{align} 
where $\phi^\star(\theta)$ indicates that the optimal $\phi^\star$ of the lower level problem depends on $\theta$, which is the variable to be optimized in the upper level problem. 
For simplicity and clarity, we denote the objectives of the upper and lower level problems as:
\begin{align}
f_G(\phi^\star(\theta))\vcentcolon=\ep_{x\sim p_{\data}(x)}[\log p_E(x;\phi^\star(\theta))],\nonumber \\
f_E(\theta, \phi)\vcentcolon=\ep_{z\sim p_Z} [\log p_E(G(z;\theta);\phi)]. \nonumber
\end{align}

We now provide the following theorem to demonstrate the correctness of LBT under the assumption that $G$ and $E$ have sufficient capacity,
which has been justified by recent advances of DGMs~\cite{kingma2013auto,oord2016pixel}.
%
%
\begin{theorem}\label{theorem:total}
	 Solving problem (\ref{eqn:obj}) is equivalent to minimizing the KL-divergence between the data distribution and the generator distribution, and it's optima is achieved when
	 \begin{align}
	     p_G = p_E = p_{\data}. \label{eqn:str-cond}
	 \end{align}
\end{theorem}
The proof is included in Appendix. Theorem~\ref{theorem:total} shows that the global optimum of LBT is achieved at $p_G=p_E=p_{\data}$ if the estimator has enough capacity.
Below, we give a further analysis to provide a weaker conclusion for LBT under a mild assumption that the estimator has only limited capacity.

{\bf Exponential family}: Consider the case where the estimator distribution $p_E(x)$ is in the exponential family form,
i.e., $p_E(x) = h(x) e^{\eta \cdot T(x) - A(\eta)}$,
where $T(x)$ denotes the sufficient statistics and $\eta$ are the natural parameters. In this case, for a certain distribution $q$, $KL(q\|p_E)$ achieves optimal iff $p_E$ captures the sufficient statistics of $q$, i.e., $\ep_{p_E} T(x) = \ep_{q} T(x)$~\cite{nasrabadi2007pattern}.
Therefore,
given $p_G$ in LBT,
the estimator distribution $p_E$ achieves optimal when $\ep_{p_E} T(x) = \ep_{p_G} T(x)$.
To make the estimator 
achieve an optimal likelihood on data samples (or equivalently, optimal $KL(p_{\data}\|p_E)$),
$G$ should ensure $E$ to capture the sufficient statistics of the data distribution, i.e., $\ep_{p_E} T(x) = \ep_{p_{\data}} T(x)$. Therefore, the estimator can still regularize $p_G$ to match $p_{\data}$ in terms of sufficient statistics:
\begin{align}\label{eqn:weak-cond}
    \ep_{p_G} T(x) = \ep_{p_{\data}} T(x),
\end{align}
which is a weaker conclusion with fewer assumptions compared to \eqn{str-cond}.
We provide an example to verify the above analysis in \secref{method:lbtgan} and demonstrate the effectiveness of an estimator beyond the exponential family on real applications in \secref{smnist}.

\vspace{-.1cm}
\subsection{Combining LBT with GAN}\label{sec:method:lbtgan}
\vspace{-.1cm}

The KL-divergence is known to be zero-avoiding~\cite{nasrabadi2007pattern} in that it encourages the model distribution to cover the data distribution.
However, in practice it may also result in low quality of generated samples~\cite{tolstikhin2017wasserstein}.
This property makes LBT complementary to GAN which tends to generate samples of high quality but lack sample diversity~\cite{goodfellow2016nips}.
To combine the best of both worlds,
we further propose to augment LBT with a discriminator as in GANs, and call the hybrid model LBT-GAN.
Formally, LBT-GAN solves the following bilevel problem:
\begin{align}\label{eqn:obj-lbtgan}
\max_\theta &\quad  f_G(\phi^\star(\theta))
-\lambda_G \cdot f_{\textrm{GAN}}(\theta,\psi^\star) \\
\textrm{s.t.}&  \quad \phi^\star(\theta) =  \argmax_{\phi} f_E(\theta, \phi), \\
& \quad \psi^\star =  \argmax_{\psi} f_{\textrm{GAN}}(\theta,\psi),
\end{align}
\noindent where
$\lambda_G$ balances the weight between two losses. 
Under the assumption that $G$ and the discriminator $D$ have sufficient capacity,
\cite{goodfellow2014generative} show that the optimum of GAN's minimax framework is achieved at $p_G=p_{\data}$, which is consistent with the conditions in 
Eqn.~\eqref{eqn:str-cond}\&\eqref{eqn:weak-cond} for LBT.
Therefore,
it is straightforward that LBT-GAN has the same global optimal solution as GAN, i.e., $p_G=p_{\data}$.

We show that LBT-GAN has advantages compared to GAN even when the estimator has only limited capacity. As mentioned above, GAN can suffer from mode collapse problem since gradient-based optimization methods could fall into a mode-collapsed local optimum of the JS-divergence. However these mode-collapsed local optima are less likely to satisfy the condition in \eqn{weak-cond}. The estimator can provide training signal to the generator and help it to escape the local optima that violates \eqn{weak-cond}, and therefore make LBT-GAN more resistant to the mode collapse problem.

\begin{figure}
    \centering
    \begin{subfigure}[t]{0.3\textwidth}
        \includegraphics[width=\textwidth]{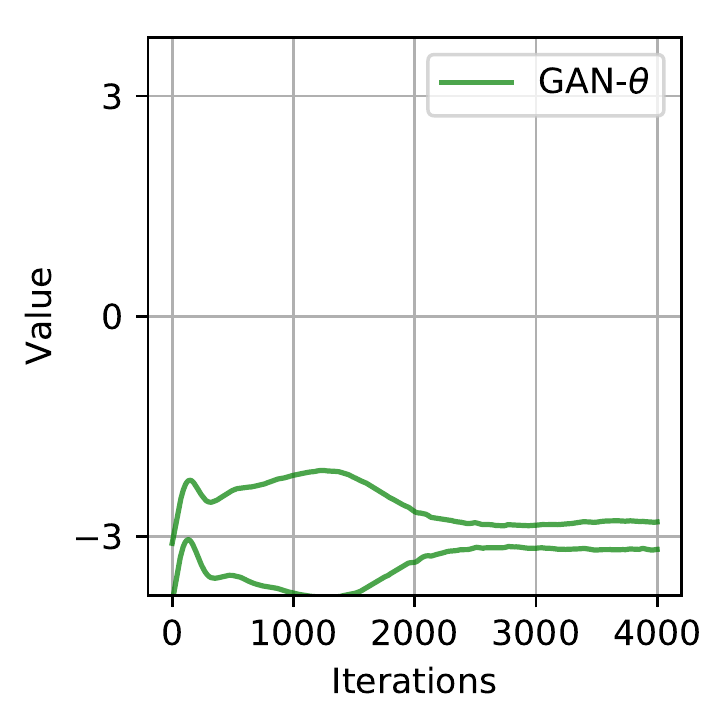}
    \end{subfigure}
    \begin{subfigure}[t]{0.3\textwidth}
        \includegraphics[width=\textwidth]{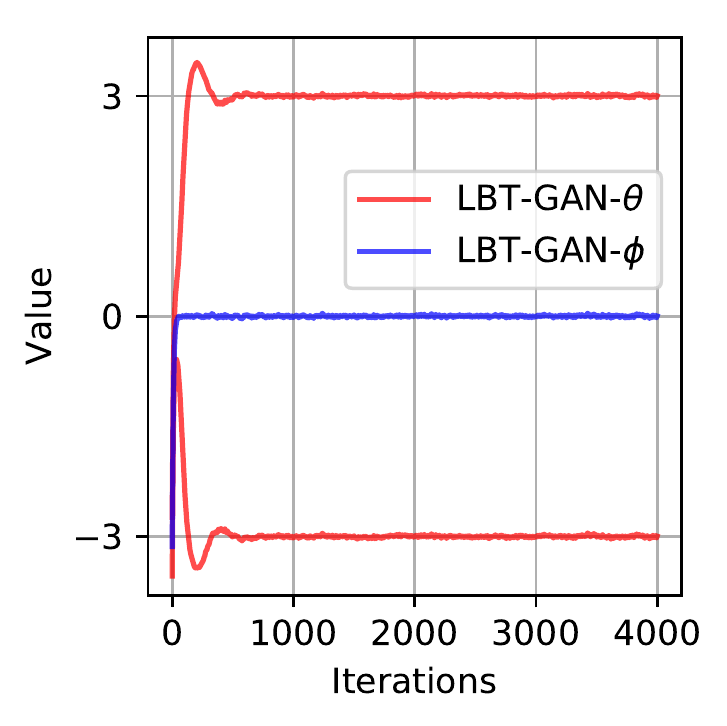}
    \end{subfigure}
    \begin{subfigure}[t]{0.3\textwidth}
        \includegraphics[width=\textwidth]{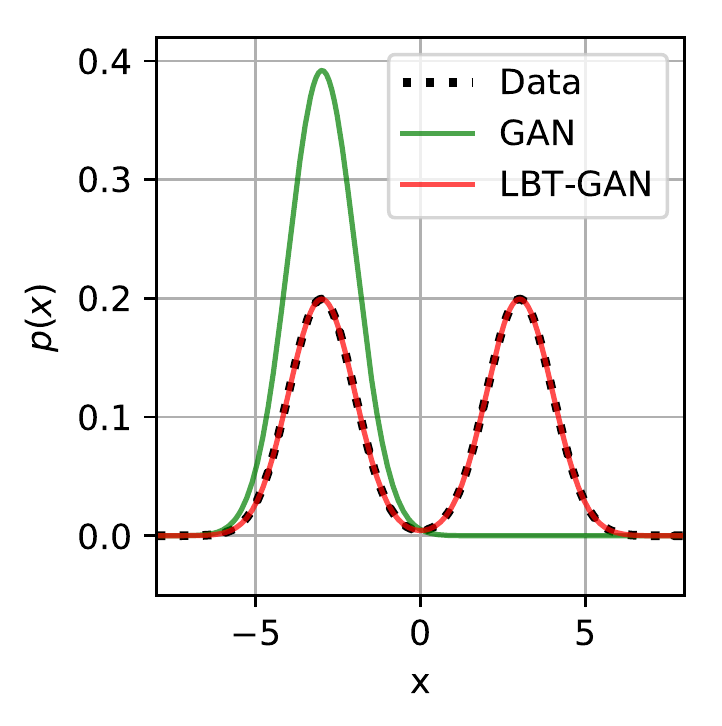}
    \end{subfigure}
	\caption{
	An illustration of LBT-GAN where the estimator has insufficient capacity. We consider the same data distribution $p_{\data}(x)$ and model distribution $p_G(x)$ as described in \fig{jsd_kl}.
	We train both GAN and LBT-GAN to learn the model $p_G(x)$. 
	For LBT-GAN, we use an estimator with insufficient capacity, i.e., $p_E(x)=\mathcal{N}(\phi, 1)$.
	The training processes of GAN (Left) and LBT-GAN (Middle) and their learned distributions (Right) are shown. We observe that while GAN learns a mode-collapsed model, LBT-GAN can escape the local optimum and capture the data distribution 
	in this case.
	}
	\label{fig:limite}
\end{figure}

To empirically verify our argument, we consider the settings of the toy example in \fig{jsd_kl}.
In LBT-GAN, we assume that the estimator $E$ is a single Gaussian $p_E(x)=\mathcal{N}(\phi, 1)$ with a learnable mean $\phi$, which can only capture the mean of a distribution. In this case,
the condition \eqn{weak-cond} ensures that $p_G$ and $p_{\data}$ have equal means. 
Therefore, if $p_G$ is around the mode-collapsed local optima of GAN 
(where the gradients of the GAN objective will be nearly zero), the gradients of the LBT objective will encourage the generator to escape the local optima with non-zero mean.
A clear demonstration is shown in \fig{limite}, where we identically initialize the means of $G$ around $-3$ in both LBT-GAN and GAN. We observe that GAN converges to a local optimum of JS-divergence, whereas LBT-GAN converges to the global optimal quickly as the estimator regularizes the generator to a distribution with zero mean.
Further experimental results on real datasets are illustrated in \secref{smnist}.

\vspace{-.1cm}
\subsection{Stochastic Gradient Ascent via Unrolling}\label{sec:method:if}
\vspace{-.1cm}

The bilevel problem is generally challenging to solve.
Here, we present a stochastic gradient ascent algorithm 
by using an unrolling technique~\cite{metz2016unrolled} to derive the gradient. For clarity, we focus on learning LBT and the methods can be directly applied to learn LBT-GAN. 
Specifically, to perform gradient ascent, we calculate the gradient of $f_G$ with respect to $\theta$ as follows:
\begin{align}
\frac{\partial f_G(\phi^\star(\theta))}{\partial \theta} & = \frac{\partial f_G(\phi^\star(\theta))}{\partial \phi^\star(\theta)}\frac{\partial \phi^\star(\theta)}{\partial \theta} \\
&=\frac{\partial f_G(\phi^\star(\theta))}{\partial \phi^\star(\theta)}\int_z \frac{\partial \phi^\star(\theta)}{\partial G(z;\theta)}\frac{\partial G(z;\theta)}{\partial \theta} p_Z dz, \nonumber
\end{align}
where both $\frac{\partial f_G(\phi^\star(\theta))}{\partial \phi^\star(\theta)}$ and $\frac{\partial G(z;\theta)}{\partial \theta}$ are easy to calculate.
However, the term $\frac{\partial \phi^\star(\theta)}{\partial G(z;\theta)}$ is intractable since $\phi^\star(\theta)$ can not be expressed as an analytic function of the generated samples $G(z;\theta)$.
We instead consider a local optimum $\hat{\phi}^\star$ of the density estimator parameters,
which can be expressed as the fixed point of an iterative optimization procedure with $\phi^0 = \phi$:
\begin{align}
\phi^{k+1} = \phi^k + \eta \cdot \left.\frac{\partial f_E(\theta,\phi)}{\partial\phi}\right|_{\phi^k} \label{eqn:unroll_general}, 
\hat{\phi}^\star = \lim_{k\rightarrow\infty}\phi^k,
\end{align}
where $\eta$ is the learning rate\footnote{We have omitted the learning rate decay for simplicity.}.
Since the samples used to evaluate the likelihood $f_E(\theta,\phi)$ are generated by $G(\cdot;\theta)$, each step of the optimization procedure is dependent on $\theta$.
We thus write $\phi^k(\theta,\phi^0)$ to clarify that $\phi^k$ is a function of $\theta$ and the initial value $\phi^0$. 
In the following, we rewrite $G(z;\theta)$ as $x_z$ and $\frac{\partial f_E(\theta, \phi)}{\partial \phi}$ as $\nabla \phi$ for simplicity.
Since $\nabla \phi$ is differentiable w.r.t.\ $x_z$ for most density estimators such as NADEs, $\phi^k(\theta,\phi^0)$ is also differentiable w.r.t.\ $x_z$.
By unrolling for $K$ steps, namely, 
using $\phi^K(\theta,\phi^0)$ to approximate $\phi^\star(\theta)$ in the objective $f_G(\phi^\star(\theta))$,
we optimize a surrogate objective formulated as $f_G(\phi^K(\theta,\phi^0))$ for the generator.
Thus, the term $\frac{\partial \phi^\star(\theta)}{\partial x_z}$ is approximated as $\frac{\partial \phi^\star(\theta)}{\partial x_z} \approx \frac{\partial \phi^K(\theta,\phi^0)}{\partial x_z}$,
which is known as the unrolling technique~\cite{metz2016unrolled}. 
Under the assumption that $\phi^0=\phi^\star$, the gradients provided by the unrolling technique are good approximations of the exact gradients. We give a formal theoretical proof in Appendix B.

Finally, the generator and the likelihood estimator can be updated using the following process:
\begin{align}
\theta \leftarrow  \theta + \eta_\theta \frac{\partial f_G(\phi^K(\theta,\phi))}{\partial \theta},~
\phi \leftarrow  \phi + \eta_\phi \frac{\partial f_E(\theta,\phi)}{\partial \phi}, 
\end{align}
where $\eta_\theta$ and $\eta_\phi$ are the learning rates for the generator and the estimator, respectively. We perform several updates of $\phi$ per update of $\theta$ to keep $p_E$ closed to $p_{G}$.
Note that for other gradient-based methods such as Adam~\cite{kingma2014adam}, the unrolling procedure is similar~\cite{metz2016unrolled}. 
In our experiments, only a few steps of unrolling, e.g., 5 steps, are sufficient.
The training procedure is described in Alg.~\ref{alg:method}.

\begin{algorithm}[h]
	\caption{Stochastic Gradient Ascent Training of LBT with the Unrolling Technique}
	\label{alg:method}
	\begin{algorithmic}
		\STATE {\bfseries Input:} data $x$, learning rate $\eta_\theta$ and $\eta_\phi$, unrolling steps $K$ and estimator update steps $M$.
		\STATE Initialize parameters $\theta_0$ and $\phi_0$, and $t=1$.
		\REPEAT
		\STATE $\phi_t^0 \leftarrow \phi_{t-1}$
		\FOR{$i=1$ {\bfseries to} $M$}
		\STATE $\phi_t^i \leftarrow \phi_t^{i-1} + \left.\eta_\phi\cdot\frac{\partial f_E(\theta,\phi)}{\partial\phi}\right|_{\phi_t^{i-1}}$ 
		\ENDFOR
		\STATE Update $\phi$: $\phi_t \leftarrow \phi_t^M$
		\STATE $\phi^0 \leftarrow \phi_{t}$
		\STATE Unrolling: $\phi^K \leftarrow \phi^0 + \left.\sum_{i=1}^K \eta_\phi\cdot\frac{\partial f_E(\theta,\phi)}{\partial\phi}\right|_{\phi^{i-1}}$
		\STATE Update $\theta$: $\theta_{t} \leftarrow \theta_{t-1} + \eta_\theta\frac{\partial f_G(\phi^K)}{\partial\theta}$
		\STATE Update $t$: $t\leftarrow t+1$
		\UNTIL{Both $\theta$ and $\phi$ converge.}
	\end{algorithmic}
\end{algorithm}


\begin{figure*}[tb]
    \centering
    \begin{subfigure}[t]{0.32\textwidth}
        \includegraphics[width=\textwidth,height=0.6\textwidth]{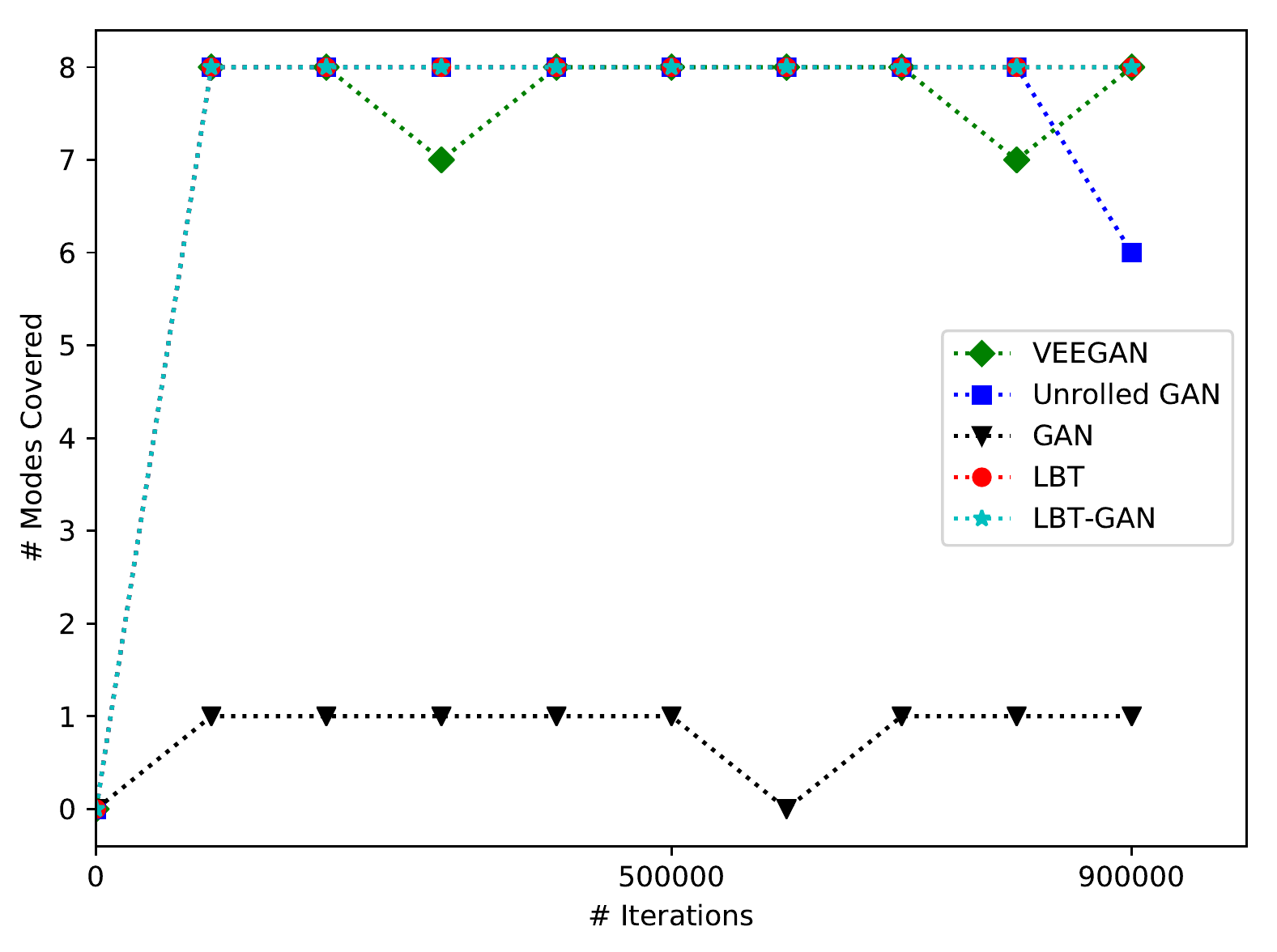}
    \end{subfigure}
    \begin{subfigure}[t]{0.32\textwidth}
        \includegraphics[width=\textwidth,height=0.6\textwidth]{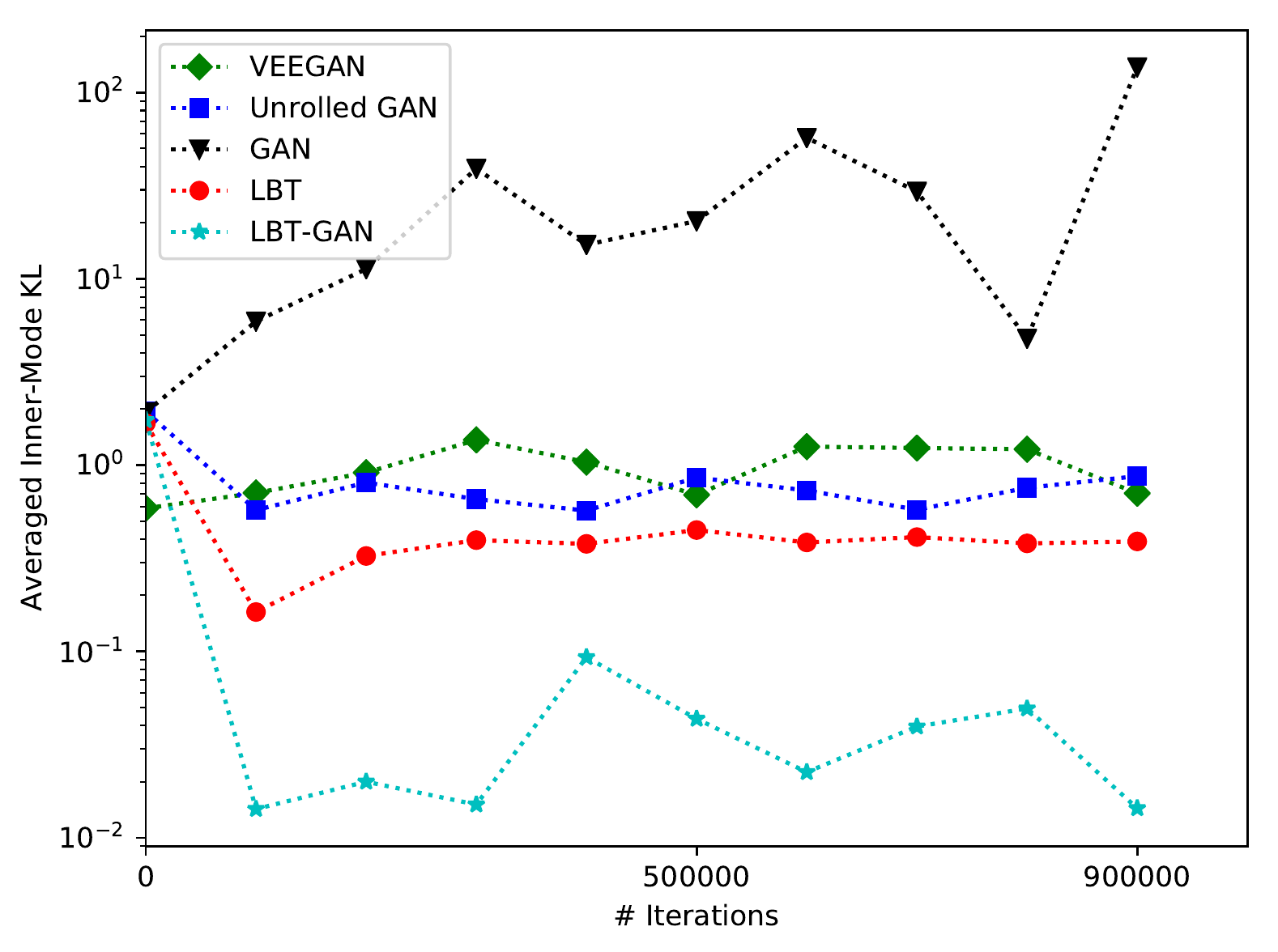}
    \end{subfigure}
    \begin{subfigure}[t]{0.32\textwidth}
        \includegraphics[width=\textwidth,height=0.6\textwidth]{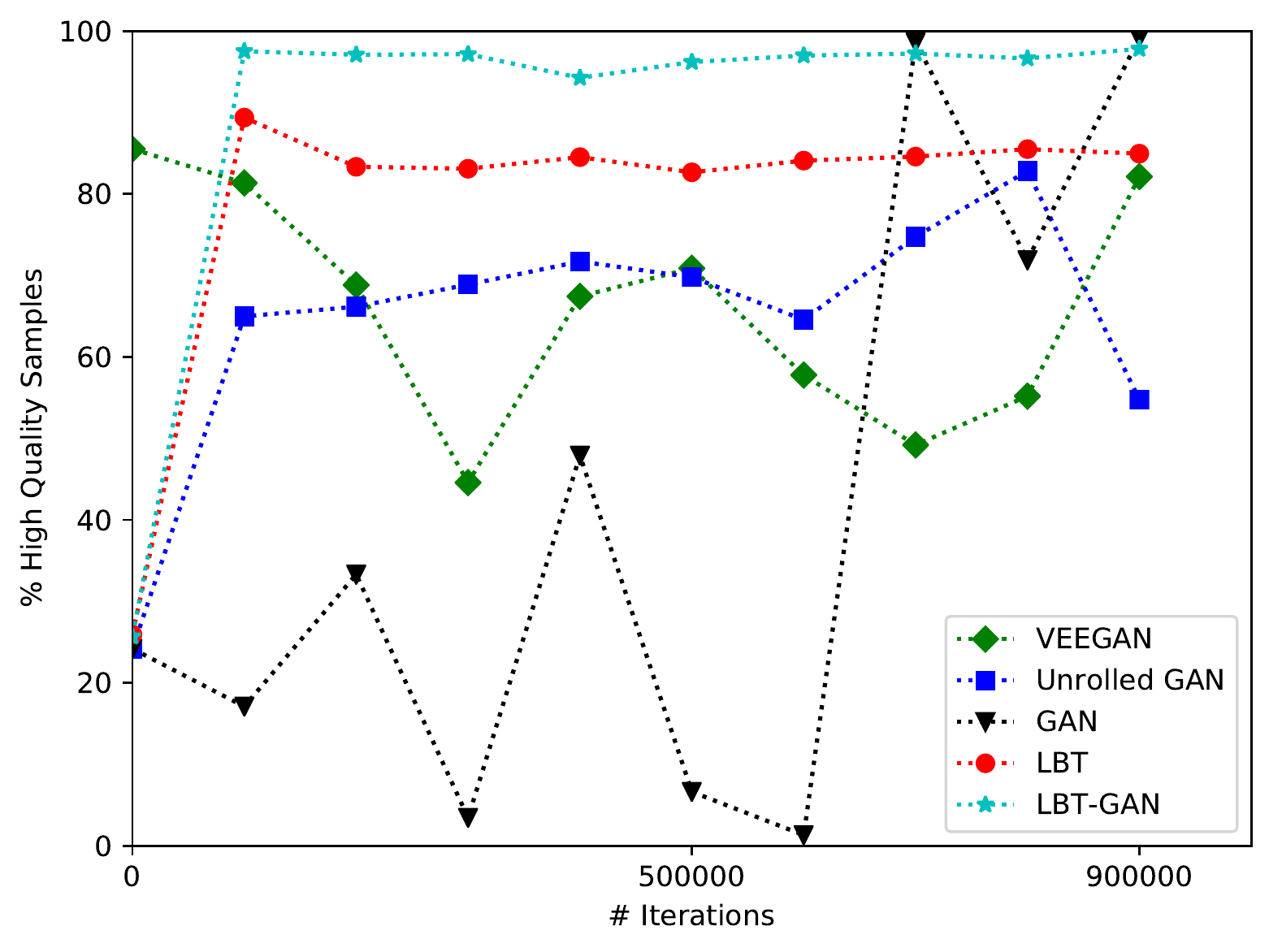}
    \end{subfigure}
    \\
    \begin{subfigure}[t]{0.32\textwidth}
        \includegraphics[width=\textwidth,height=0.6\textwidth]{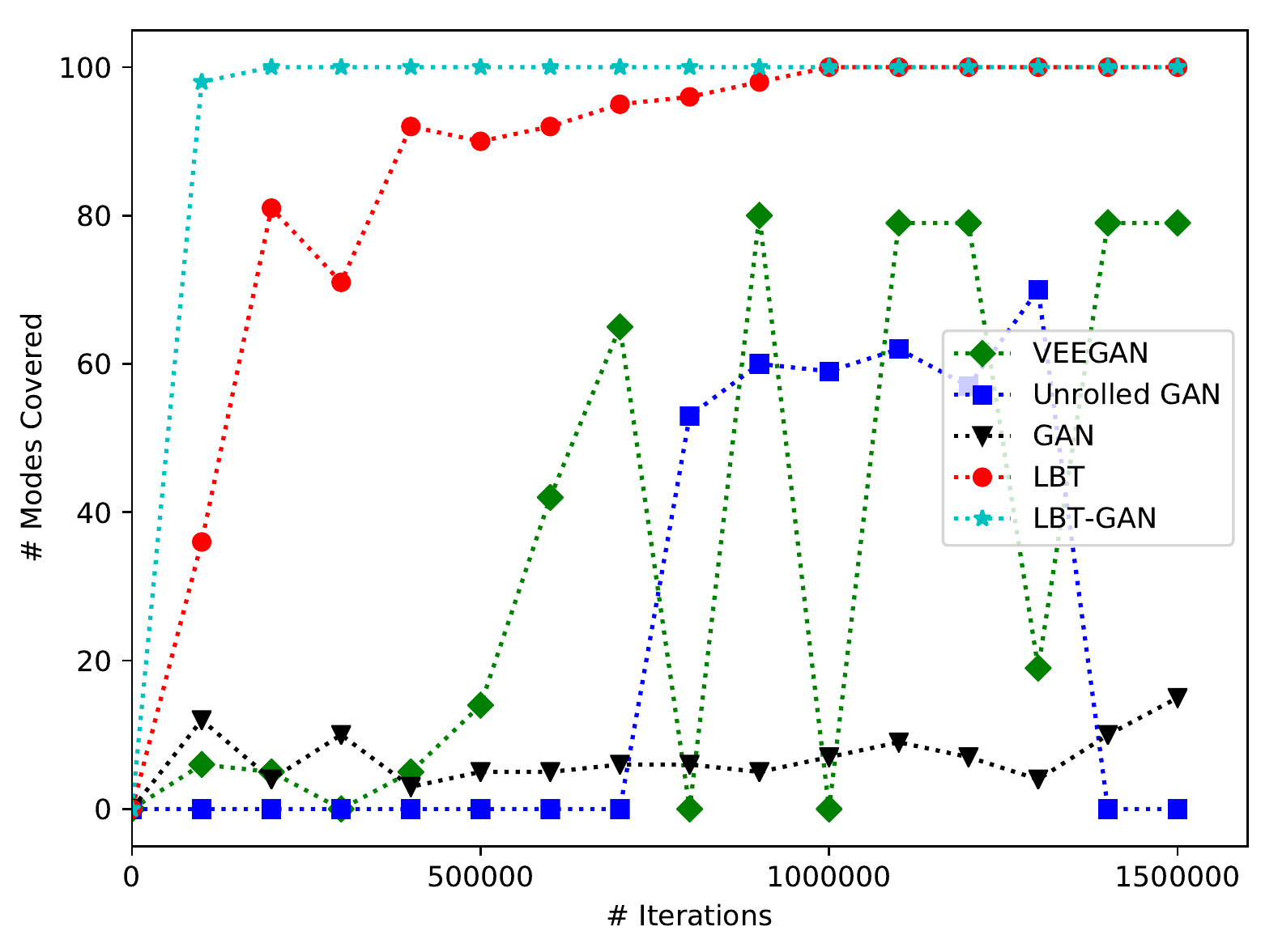}
        \caption{\# Modes Covered}\label{fig:toy-curve:grid:mode}
    \end{subfigure}
    \begin{subfigure}[t]{0.32\textwidth}
        \includegraphics[width=\textwidth,height=0.6\textwidth]{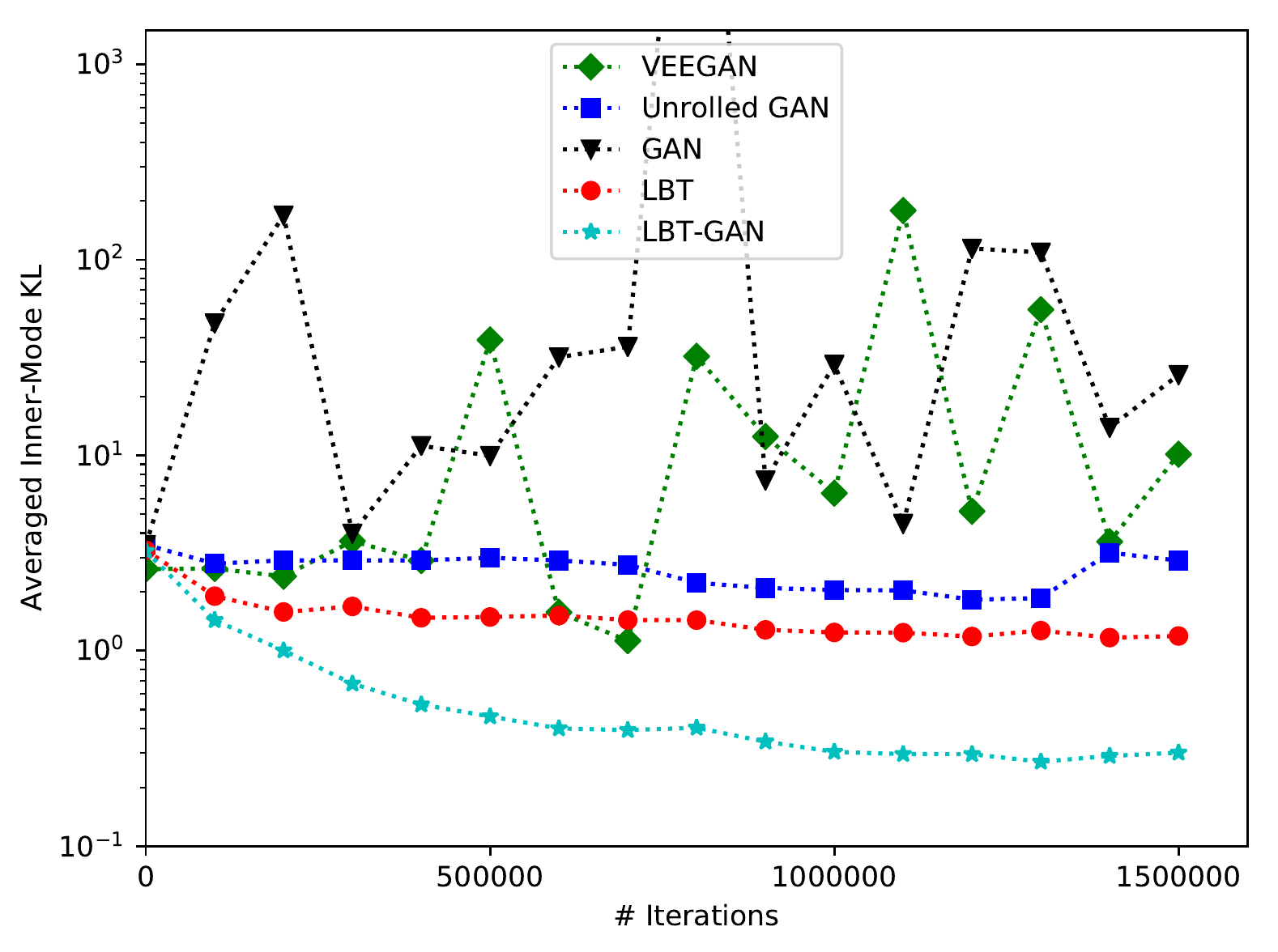}
        \caption{Intra-mode KL-divergence}\label{fig:toy-curve:grid:kl}
    \end{subfigure}
    \begin{subfigure}[t]{0.32\textwidth}
        \includegraphics[width=\textwidth,height=0.6\textwidth]{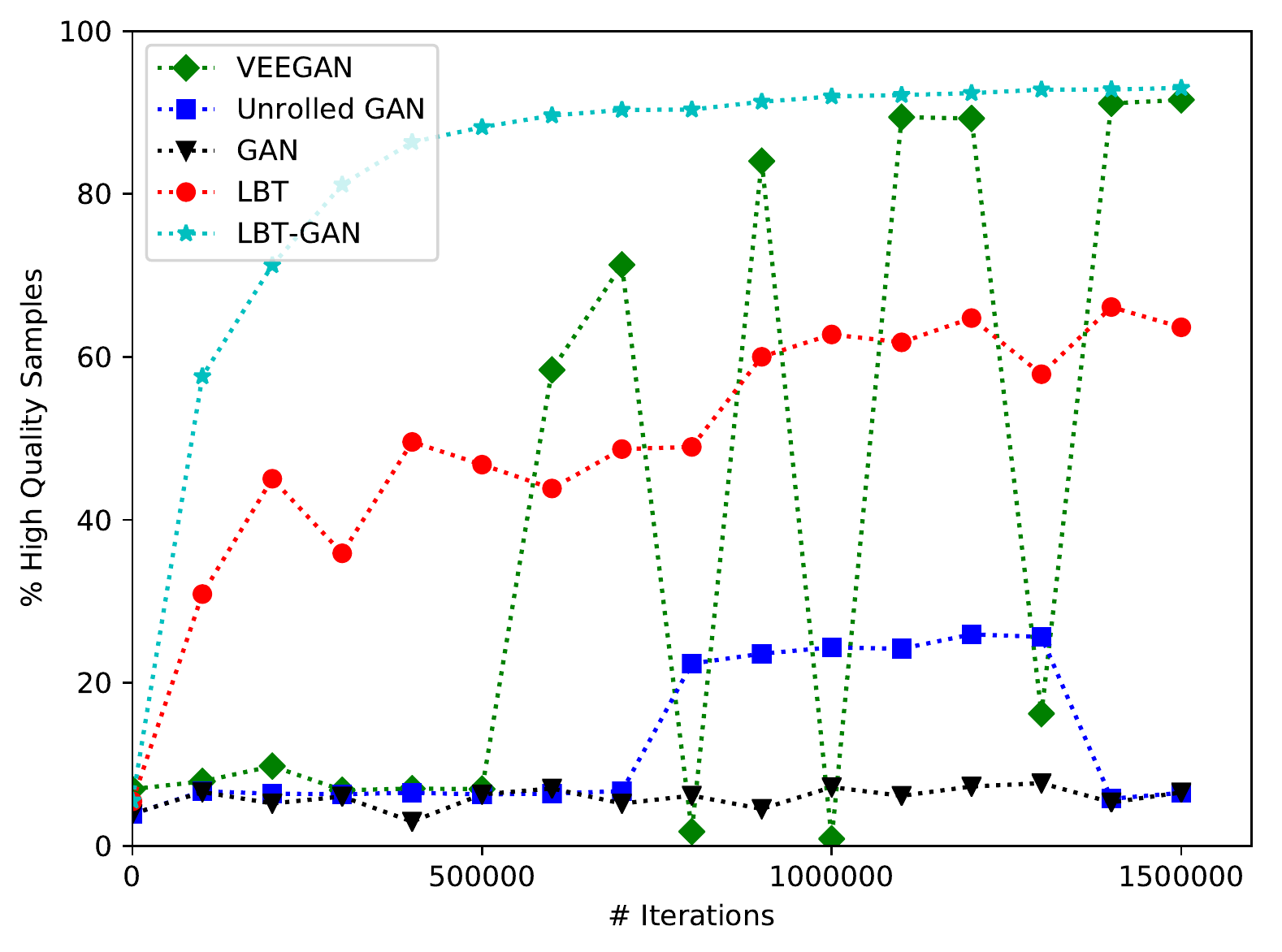}
        \caption{\% of high quality samples}\label{fig:toy-curve:grid:rate}
    \end{subfigure}
	\caption{Three different metrics evaluated on the generator distributions of different methods trained on the ring data (Top) and the grid data (Bottom). The metrics from left to right are: Number of modes covered (the higher the better); Averaged intra-mode KL-divergence (the lower the better); Percentage of high quality samples (the higher the better).}\label{fig:toy-curve}
    \vspace{-.2cm}
\end{figure*}

\vspace{-.1cm}
\section{Related Work}
\vspace{-.1cm}

Implicit statistical models~\cite{mohamed2016learning} are of great interests with the emergence of GAN~\cite{goodfellow2014generative} that introduces a minimax framework to train such models.
\cite{nowozin2016f} generalize the original GANs via introducing a broad class of $f$-divergence for optimization. In comparison, LBT provides a different way to optimize the KL-divergence and achieves good results on avoiding mode collapse and generating realistic samples when combined with GAN. \cite{arjovsky2017wasserstein} propose to minimize the earth mover's distance to avoid the problem of gradient vanishing in vanilla GANs. 
Besides, \cite{li2015generative} train implicit models by matching momentum between generated samples and real samples. 

Mode collapse is a well-known problem in practical training of GANs.
Much work has been done to 
alleviate the problem~\cite{arjovsky2017towards,arjovsky2017wasserstein,mao2017least,metz2016unrolled,srivastava2017veegan}.
Unrolled GAN~\cite{metz2016unrolled} proposes to unroll the update of the discriminator in GANs. The unrolling helps capturing how the discriminator would react to a change in the generator. Therefore it reduces the tendency of the generator to collapse all samples into a single mode.
VEEGAN~\cite{srivastava2017veegan} introduces an additional reconstructor net to map data back to the noise space. Then a discriminator on the joint space is introduced to learn the joint distribution, similar as in ALI~\cite{dumoulin2016adversarially}.
\cite{lin2017pacgan} propose to modify the discriminator to distinguish whether multiple samples are real or generated.

Different from methods in \cite{nowozin2016f,nguyen2017dual}, our method evaluates the KL-divergence with data samples and makes $\phi^\star$ a function of $\theta$ via unrolling.
This enables us to accurately evaluate the KL-divergence, regardless of whether the generator collapses or not.
By unrolling the optimization process of $\phi$, the estimation of the KL-divergence can be differentiable w.r.t. the generator and can be optimized in practice.
We directly compare our methods with the above related methods in \secref{experiments}.

VAE-GAN~\cite{larsen2015autoencoding} also conjoins a likelihood-based objective and GAN, by utilizing the encoder/decoder structure and defining the reconstruction term in the feature space of a discriminator to generalize the metric of similarity.
Whereas in LBT and LBT-GAN, an auxiliary estimator is introduced to learn the model distribution and acts as a surrogate distribution of the implicit model. This enables LBT and LBT-GAN to be combined with any (approximate) density estimators.

\vspace{-.1cm}
\section{Experiments}\label{sec:experiments}
\vspace{-.1cm}

We now present the experimental results of LBT and LBT-GAN on both synthetic and real datasets.
Throughout the experiments, we use Adam~\cite{kingma2014adam} with the default setting to  optimize both the generator and the estimator (and the discriminator for LBT-GAN). We set the unrolling steps $K=5$. 
We perform $M=15$ steps of estimator update after each generator update. 
In LBT-GAN, We choose $\lambda_G$ from $\{0.1, 1, 5, 10\}$ manually.
Our codes will be released after the double-blind review process.

\vspace{-.1cm}
\subsection{Synthetic Datasets}
\vspace{-.1cm}

\begin{figure}[t]
	\centering
	\includegraphics[width=0.23\textwidth]{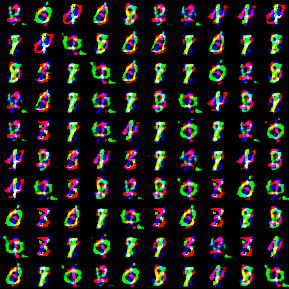}
	\includegraphics[width=0.23\textwidth]{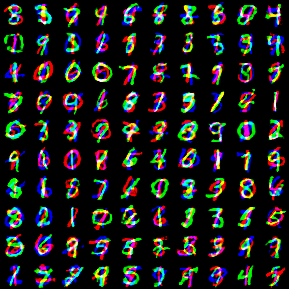}
	\includegraphics[width=0.23\textwidth]{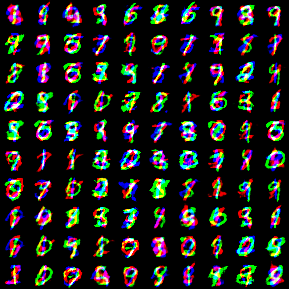}
	\includegraphics[width=0.23\textwidth]{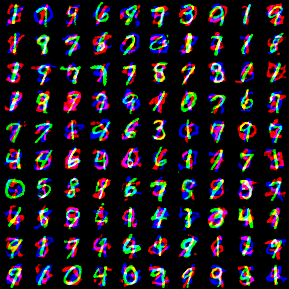}
	\caption{Generated samples of DCGANs and LBT-GANs with different network architectures of discriminators.
		From left to right: DCGAN and LBT-GAN with a large discriminator; DCGAN and LBT-GAN with a small discriminator. LBT-GANs can successfully generate realistic and diverse samples with different network architectures of discriminators.
	}
	\label{fig:smnist}
\end{figure}

 \begin{figure}[t]
 	\centering
	\includegraphics[width=0.23\textwidth]{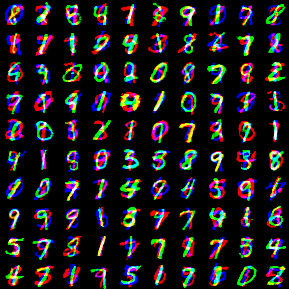}
	\includegraphics[width=0.23\textwidth]{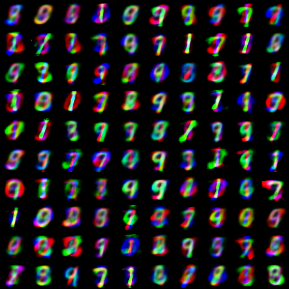}
	\caption{The generated samples of LBT-GAN (Left) with the estimator being a smaller VAE and the samples from the estimator (Right).}
	\label{fig:simplevae}
	\vspace{-.4cm}
\end{figure}

%

\begin{table*}[tbp]
\vspace{-.4cm}
  \centering
  \caption{Degree of mode collapse measured by number of mode captured (\# MC) and KL-divergence between the generated distribution over modes and the uniform distribution over $1,000$ modes on Stacked MNIST. Results are averaged over 5 runs.}
   \begin{small}
    \begin{tabular}{ccccc}
    \hline
            & ALI   & Unrolled~GAN & VEEGAN & DCGAN  \\
    \hline
    \# MC &  $16$    & $48.7$  & $150$   & $188.8$ \\
    KL    &  $5.4$   & $4.32$  & $2.95$  & $3.17$   \\
    \hline
           & PacGAN & D2GAN &  LBT-GAN (VAE) & LBT-GAN (NADE) \\
    \hline
    \# MC  & $664.2$ & $876.8$ & $\mathbf{999.6}$ & $\mathbf{1000}$ \\
    KL    & $1.41$ & $0.95$ & $\mathbf{0.19}$ & $\mathbf{0.05}$ \\
    \hline
    \end{tabular}
    \end{small}
    \label{table:mode}%
    \vspace{-.2cm}
\end{table*}%

We first compare LBT and LBT-GAN with  state-of-the-art competitors~\cite{goodfellow2014generative,mao2017least,metz2016unrolled,srivastava2017veegan} on 2-dimensional (2D) synthetic datasets, which are convenient for qualitative and quantitative analysis.
Specifically, we construct two datasets:
(i) \textbf{ring}: mixture of $8$ 2D Gaussian distributions arranged in a ring
and (ii) \textbf{grid}: mixture of $100$ 2D Gaussian distributions arranged in a 10-by-10 grid.
All of the mixture components are isotropic Gaussian, i.e., with diagonal covariance matrix.
For the ring data, the deviation of each Gaussian component is $\textrm{diag}(0.1, 0.1)$ and the radius of the ring is $1$
\footnote{In the original Unrolled GAN's setting~\cite{metz2016unrolled}, the std of each component is $0.02$ and the radius of the ring is $2$. In our setting, the ratio of std to radius is 10 times larger. We choose this setting in order to characterize different performance of ``Intra-mode KL-divergence'' clearly.}.
For the grid data, the spacing between adjacent modes is $0.2$ and the deviation of each Gaussian component is $\textrm{diag}(0.01, 0.01)$.
\fig{toy:grid:data}
shows the true distributions of the ring data and the grid data, respectively.
For all experiments on synthetic data, we use variational auto-encoders (VAEs) as the (approximate) density estimators for both LBT and LBT-GAN.
All the encoders and decoders in VAEs are two-hidden-layer MLPs.
For fair comparison, we use generators with the same network architectures (two-hidden-layer MLPs) for all methods.
For GAN-based methods, the discriminators are also two-hidden-layer MLPs.
The numbers of the hidden units for the generators and the estimators (and the discriminators for LBT-GAN) are all $128$.

To quantify the quality of the generator learned by different methods, we report the following 3 metrics to demonstrate different characteristics of generator distributions.
\textbf{Percentage of High Quality Samples}~\cite{srivastava2017veegan}:  
We count a sample as a \textit{high quality} sample of a mode if it is within three standard deviations of that mode. We say a sample is of high quality if it is a high quality sample of any modes.
We generate $500,000$ samples from each method and report the percentage of high quality samples.
\textbf{Number of Modes Covered}:
We count a mode as a \textit{covered mode} if the number of its high quality samples is greater than $20\%$ of the expected number of that, i.e., $20\% \times \frac{\textrm{\# of samples}}{\textrm{\# of modes}}$.
\footnote{
The exact expected number
should be a little bit less than $\frac{\textrm{\# of samples}}{\textrm{\# of modes}}$, according to the three-sigma rule.}
Intuitively, lower number of modes covered indicates more severe of mode collapse and a lack of global diversity.
\textbf{Averaged Intra-Mode KL-Divergence}:
We assign each generated sample to the nearest mode of the true distribution. For each mode, we fit a Gaussian model on its assigned samples, which can be viewed as an estimate of the generator distribution at that mode (where the true distribution is approximately Gaussian).
We define \textit{intra-mode KL-divergence} as the KL-divergence between the true distribution and the estimated distribution at each mode. Intuitively, it measures the local mismatch between the generator distribution and the true one.
The averaged intra-mode KL-divergence over all modes are reported.

\fig{toy} shows the generator distributions learned by different methods. Each distribution is plotted using kernel density estimation with $500,000$ samples.
We can see that LBT and LBT-GAN manage to cover the largest number of modes on both ring and grid datasets compared to other methods, demonstrating that LBT can generate globally diverse samples. The quantitative results are included in \fig{toy-curve:grid:mode}.
Note that our method covers all the 100 modes on the grid dataset while the best competitors LSGAN and VEEGAN cover 88 and 79 modes respectively.
Moreover, the number of modes covered by LBT increases consistently. On the contrary, Unrolled GAN and VEEGAN can sometimes drop the covered modes, 
attributed to their unstable training.

\fig{toy-curve:grid:kl}
shows the results of averaged intra-mode KL-divergence. We can see that LBT and LBT-GAN consistently outperform other competitors, which demonstrates that LBT framework can help capture better intra-mode diversity.
According to 
\fig{toy:grid:lsgan} and \fig{toy:grid:veegan},
although LSGAN and VEEGAN can achieve good mode coverage, they tend to concentrate most of the density near the mode means 
and fail to capture the local diversity within each mode.
In LBT-GAN, the discriminator has a similar effect, while the estimator prevents the generator to over-concentrate the density. Therefore, the Intra-mode KL-divergence of LBT-GAN may oscillate during training as in \fig{toy-curve:grid:kl}.

Finally, we show the percentages of high quality samples for each method in \fig{toy-curve:grid:rate}.
We find that LBT-GAN achieves better results than LBT and outperforms other competitors.
As LBT-GAN can
generate high quality samples while maintaining the global and local mode coverage, we use LBT-GAN in the following experiments.



\begin{figure*}[!tb]
    \centering
    \begin{subfigure}[t]{0.13\textwidth}
        \includegraphics[width=\textwidth]{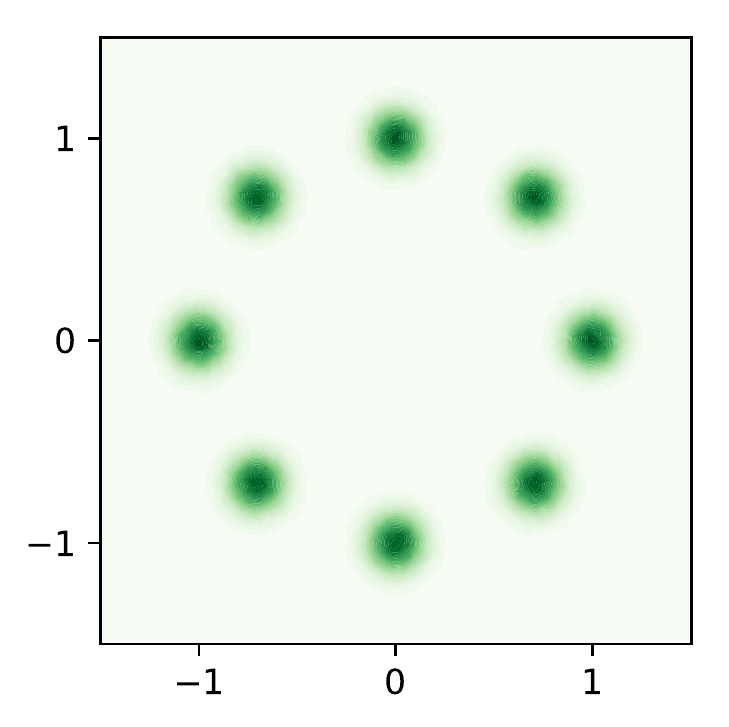}
    \end{subfigure}
    \begin{subfigure}[t]{0.13\textwidth}
        \includegraphics[width=\textwidth]{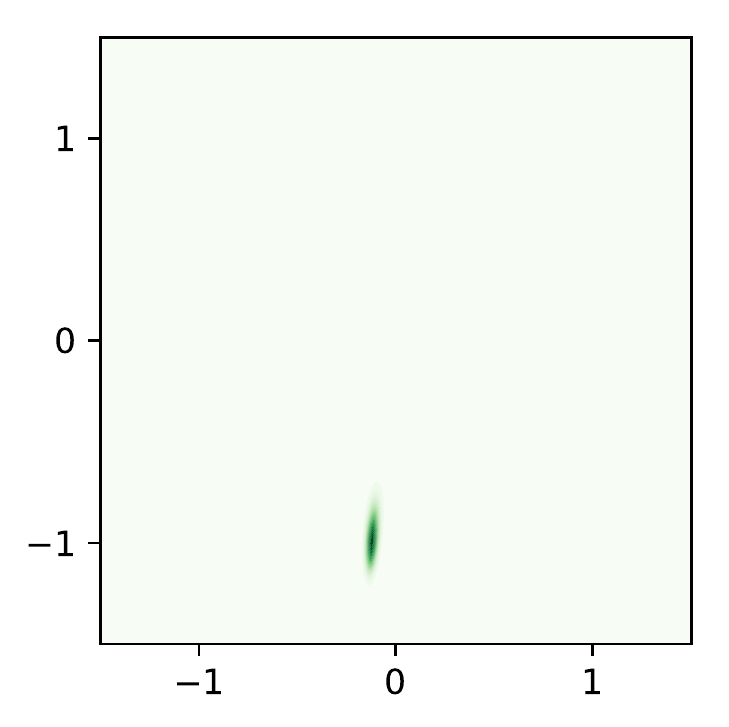}
    \end{subfigure}
    \begin{subfigure}[t]{0.13\textwidth}
        \includegraphics[width=\textwidth]{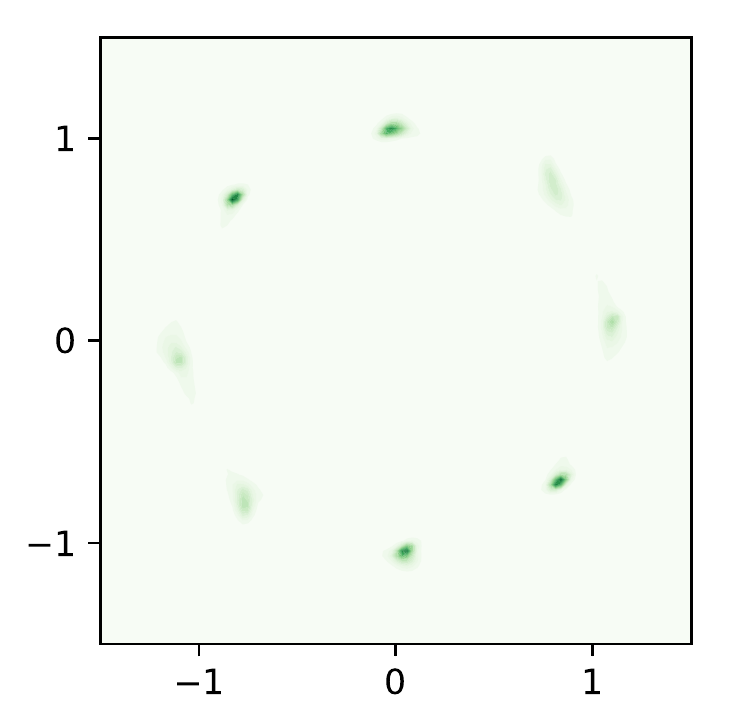}
    \end{subfigure}
    \begin{subfigure}[t]{0.13\textwidth}
        \includegraphics[width=\textwidth]{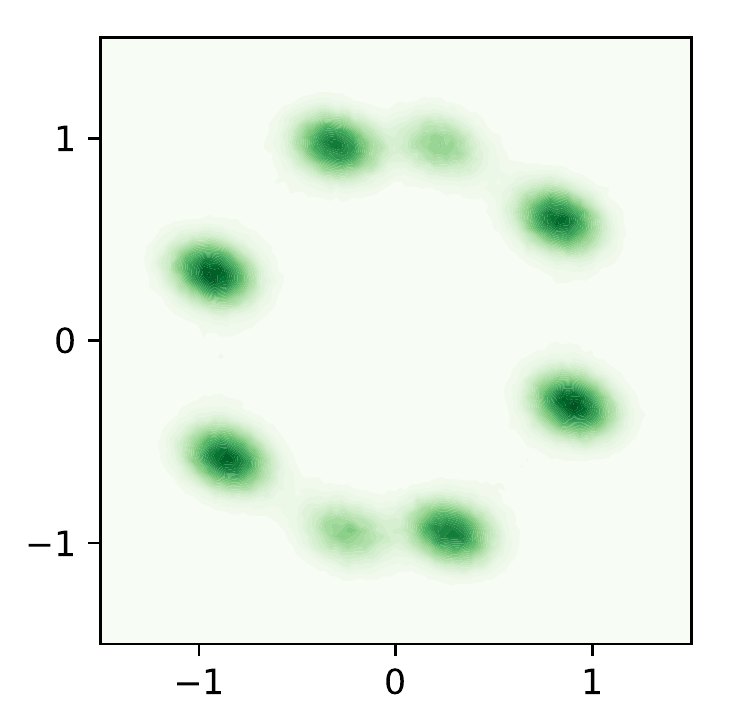}
    \end{subfigure}
    \begin{subfigure}[t]{0.13\textwidth}
        \includegraphics[width=\textwidth]{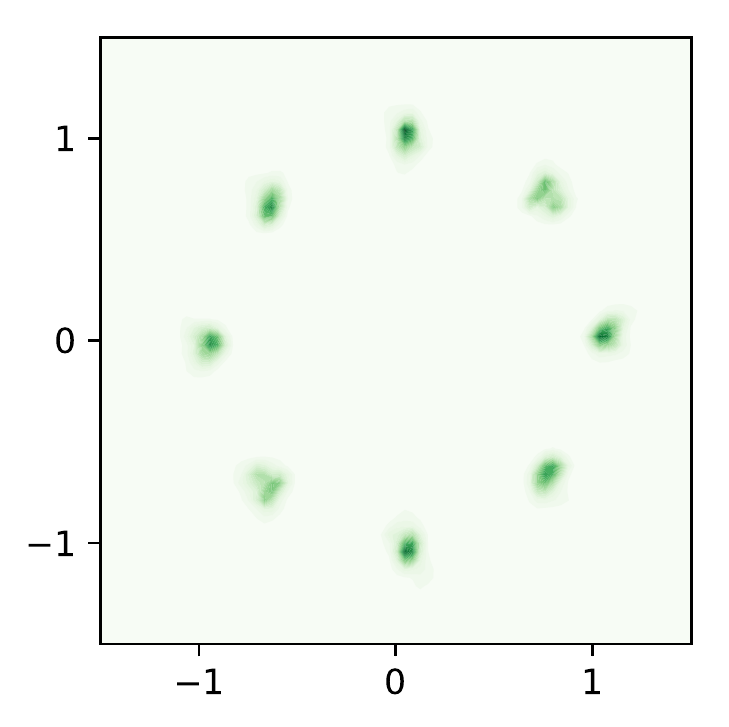}
    \end{subfigure}
    \begin{subfigure}[t]{0.13\textwidth}
        \includegraphics[width=\textwidth]{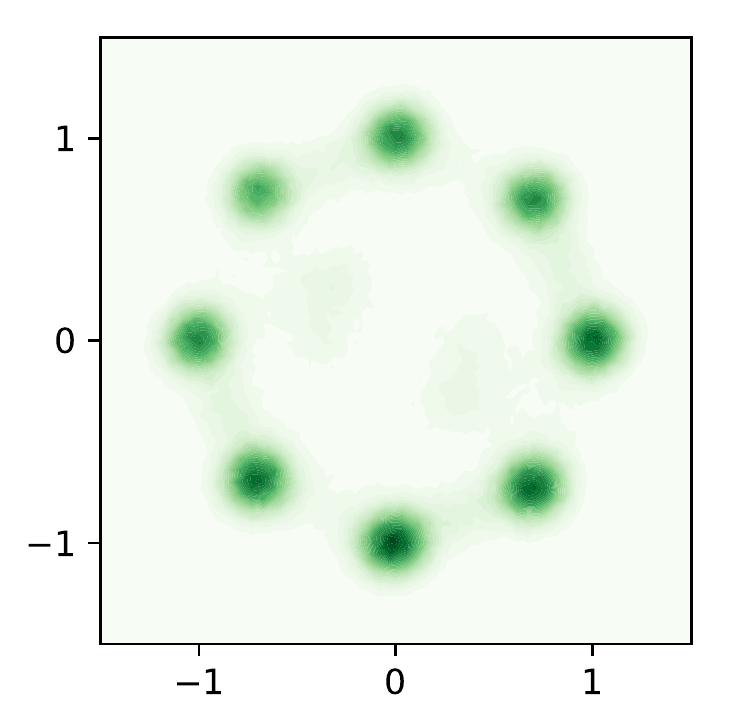}
    \end{subfigure}
    \begin{subfigure}[t]{0.13\textwidth}
        \includegraphics[width=\textwidth]{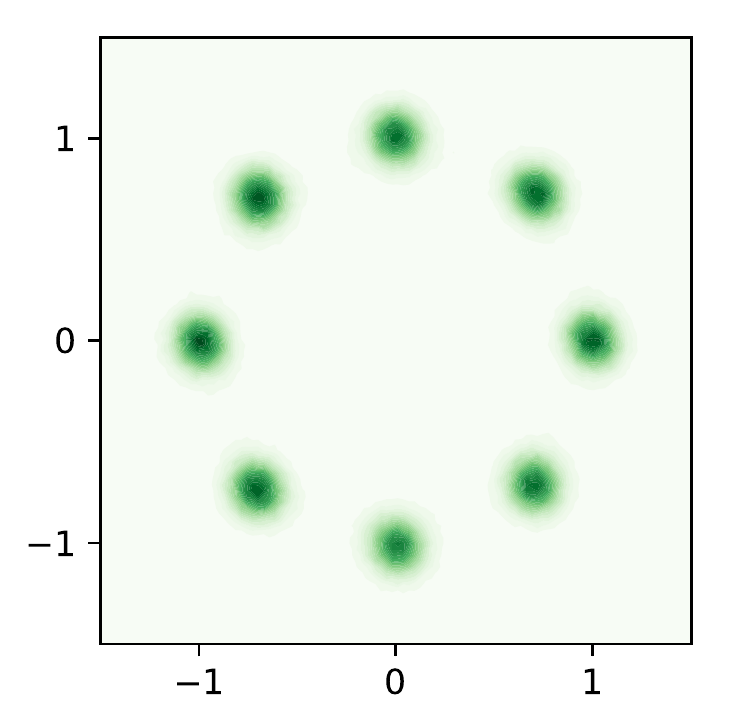}
    \end{subfigure}
    \\
    \begin{subfigure}[t]{0.13\textwidth}
        \includegraphics[width=\textwidth]{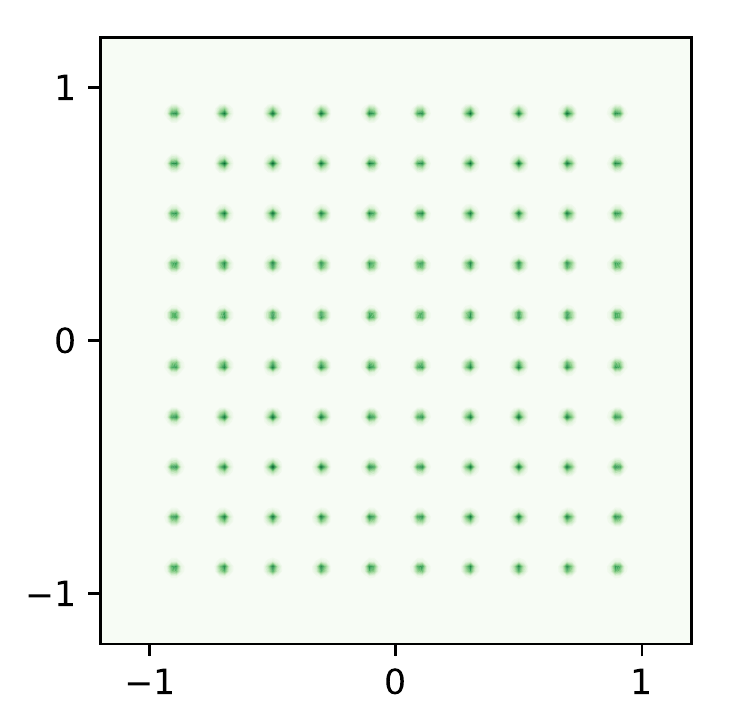}
        \caption{}\label{fig:toy:grid:data}
    \end{subfigure}
    \begin{subfigure}[t]{0.13\textwidth}
        \includegraphics[width=\textwidth]{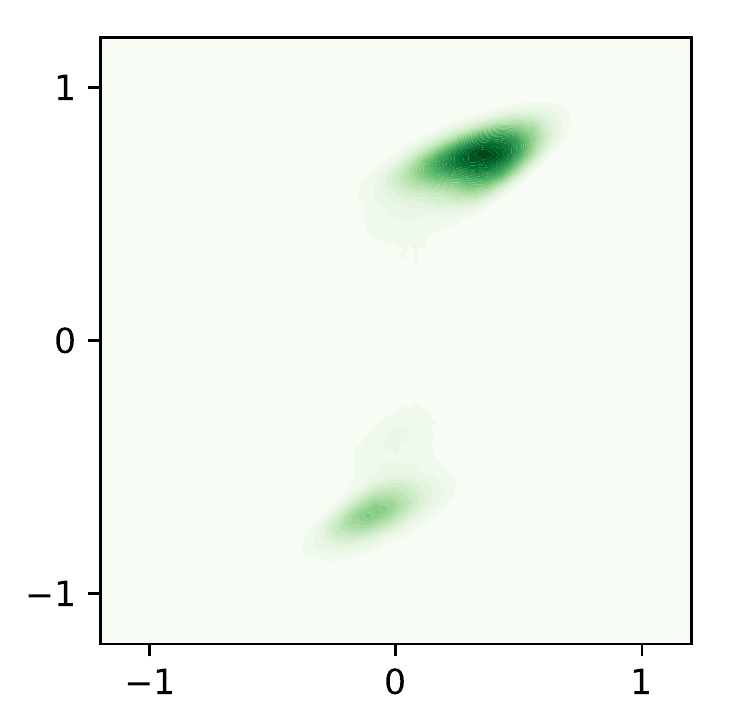}
        \caption{}\label{fig:toy:grid:gan}
    \end{subfigure}
    \begin{subfigure}[t]{0.13\textwidth}
        \includegraphics[width=\textwidth]{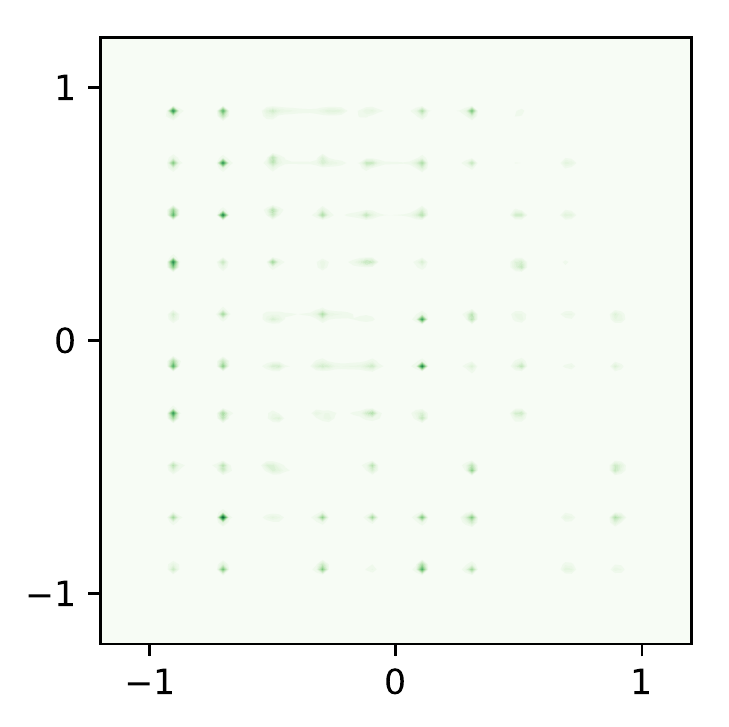}
        \caption{}\label{fig:toy:grid:lsgan}
    \end{subfigure}
    \begin{subfigure}[t]{0.13\textwidth}
        \includegraphics[width=\textwidth]{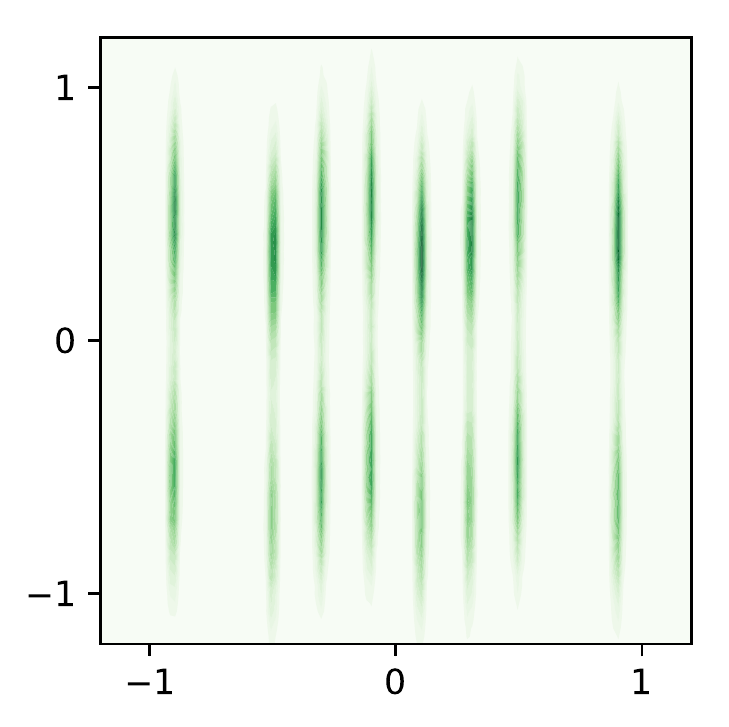}
        \caption{}\label{fig:toy:grid:urgan}
    \end{subfigure}
    \begin{subfigure}[t]{0.13\textwidth}
        \includegraphics[width=\textwidth]{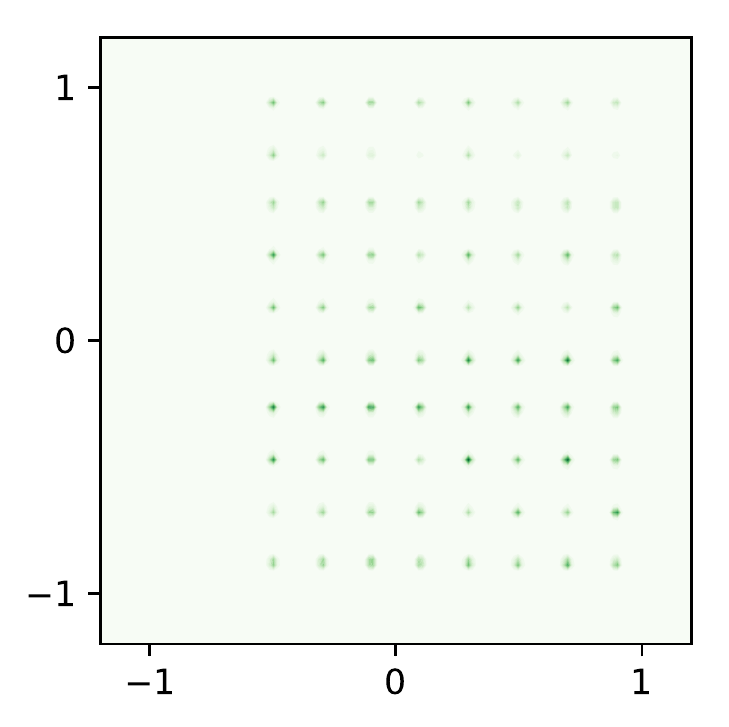}
        \caption{}\label{fig:toy:grid:veegan}
    \end{subfigure}
    \begin{subfigure}[t]{0.13\textwidth}
        \includegraphics[width=\textwidth]{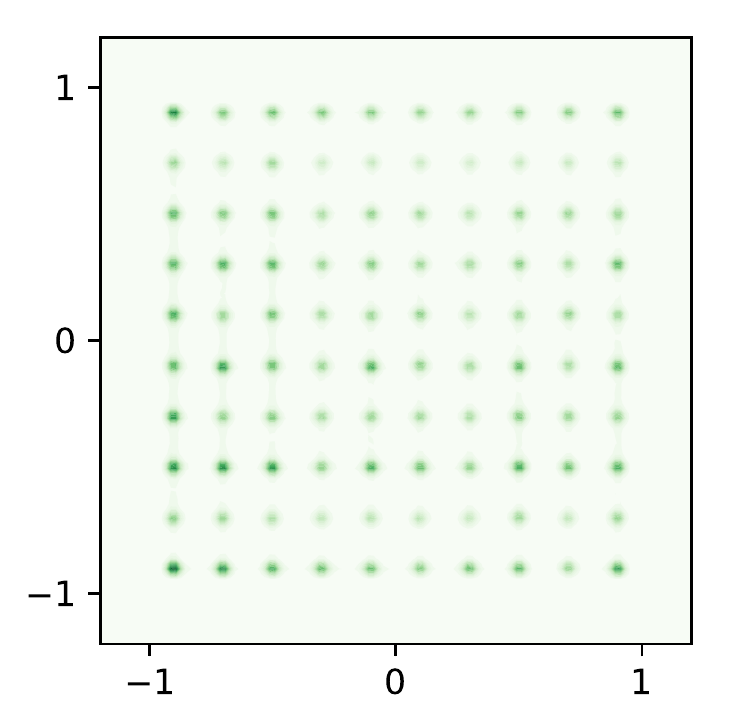}
        \caption{}\label{fig:toy:grid:our}
    \end{subfigure}
    \begin{subfigure}[t]{0.13\textwidth}
        \includegraphics[width=\textwidth]{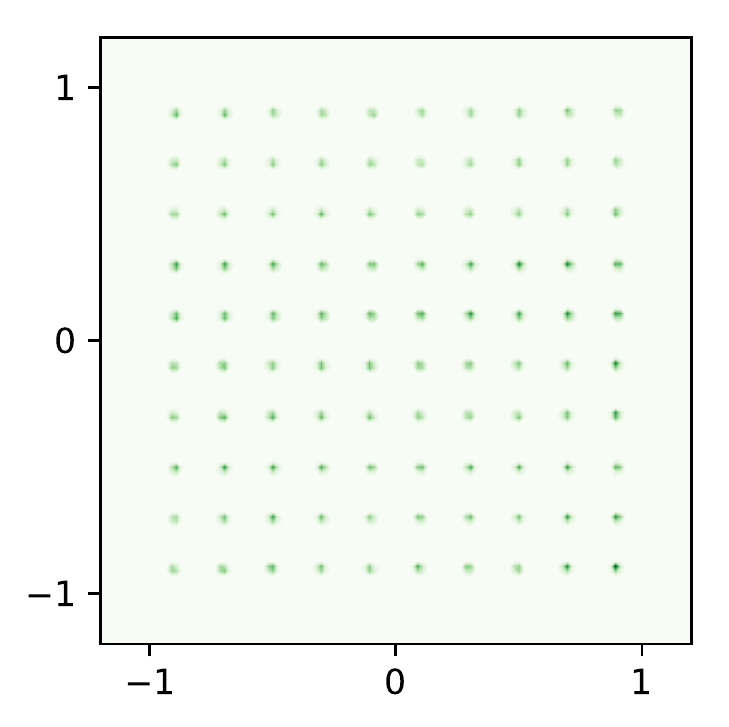}
        \caption{}\label{fig:toy:grid:lbtgan}
    \end{subfigure}
	\caption{Density plots of the data and the generator distributions trained on the ring data (Top) and the grid data (Bottom). The figures denote: (a) data, (b) vanilla GAN, (c) LSGAN, (d) Unrolled-GAN, (e) VEE-GAN, (f) LBT, (g) LBT-GAN. }\label{fig:toy}
    \vspace{-.2cm}
\end{figure*}

\begin{figure*}[tb]
    \centering
    \begin{subfigure}[t]{0.45\textwidth}
        \includegraphics[width=0.49\textwidth]{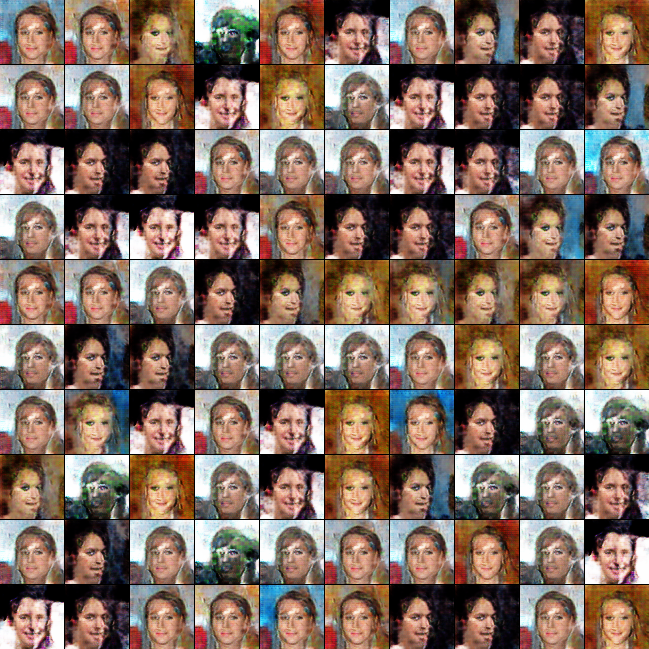}
        \includegraphics[width=0.49\textwidth]{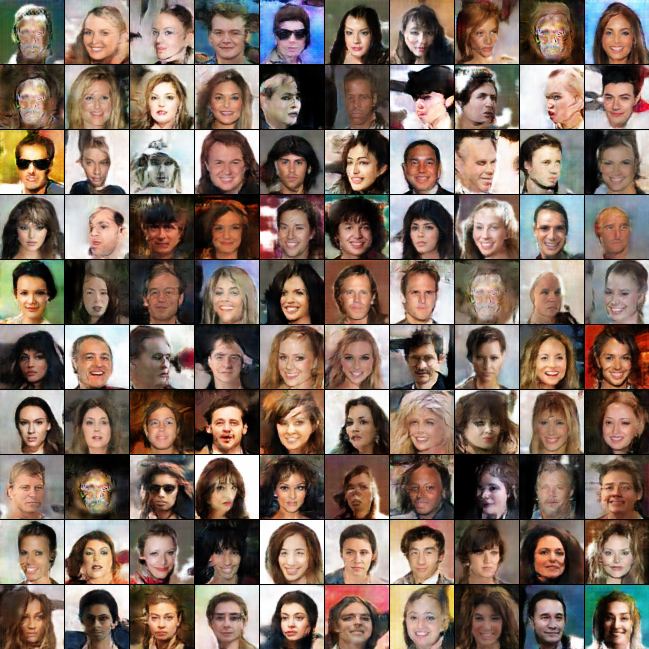}
        \caption{CelebA: DCGAN (Left) and LBT-GAN (Right).}
    \end{subfigure}
	\hspace{.2cm}
    \begin{subfigure}[t]{0.45\textwidth}
        \includegraphics[width=0.49\textwidth]{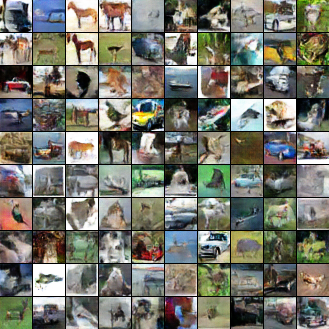}
        \includegraphics[width=0.49\textwidth]{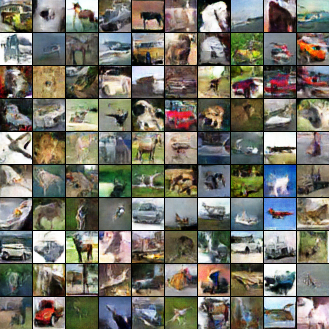}
        \caption{CIFAR10: DCGAN (Left) and LBT-GAN (Right).}
    \end{subfigure}
    \caption{Generated samples on CelebA (a) and CIFAR10 (b) of DCGANs and LBT-GANs.}
    \label{fig:color_image}
    \vspace{-.4cm}
\end{figure*}

\vspace{-.1cm}
\subsection{Stacked Mnist}\label{sec:smnist}
\vspace{-.1cm}

Stacked MNIST~\cite{metz2016unrolled} is a variant of the MNIST~\cite{lecun1998gradient} dataset created by stacking three randomly selected digits along the color channel to increase the number of discrete modes. There are $1,000$ modes corresponding to $10$ possible digits in each channel. Following~\cite{metz2016unrolled,srivastava2017veegan}, we randomly stack $128,000$ samples serving as the training data
and use $26,000$ generated samples to calculate the number of modes to which at least one sample belongs. 
We use a classifier trained on the original MNIST to identify digits in each channel of generated samples. 
Besides,
we also report the KL-divergence between the generated distribution over modes and the uniform distribution. Since carefully fine-tuned GANs can generate $1,000$ modes~\cite{metz2016unrolled}, we use smaller convolutional networks as both the generator and discriminator making our setting comparable to the competitors. 
We use two different density estimators for LBT-GAN: (i) a VAE with two-hidden-layer MLP ($1000$-$400$ hidden units) decoder and encoder
and (ii) a NADE with a single hidden layer of $50$ hidden units.

\tabl{mode} presents the quantitative results. 
In terms of the number of captured modes, LBT-GAN surpasses other competitors, which demonstrates the effectiveness of the LBT framework. Specifically, LBT-GAN can successfully capture almost all modes, and the results of KL-divergence indicate that the distribution of LBT-GAN over modes is much more balanced compare to other competitors.
Moreover, LBT-GAN works well with both approximate density estimators (e.g., VAE) and tractable density estimators (e.g., NADE).
Note that PacGAN and D2GAN also report comparable results with ours on different network architectures whereas they fail to capture all modes in our setting.
In contrast, LBT-GAN can generalize to PacGAN's architecture and capture all 1000 modes.
Our hypothesis is that the auxiliary estimators helps LBT-GAN generalize across different architectures.

\fig{smnist} shows the generated samples of GANs and LBT-GANs with different size of discriminators.
The visual quality of the samples generated by LBT-GANs is better than GANs. 
Further, we find the sample quality of DCGANs is sensitive to the size of the discriminators, while LBT-GANs can generate high-quality samples under different network architectures.

Furthermore, we implement LBT-GAN with a much smaller VAE where both the encoder and the decoder of the estimator are two-hidden-layer MLPs with only 20 units in each hidden layer. \fig{simplevae} shows that the samples from this VAE are of poor quality, which means that it can hardly capture the distribution of Stacked-MNIST.
Nevertheless, even with such a simple VAE, LBT-GAN can still capture $1,000$ modes and generate visually realistic samples (See the left panel of~\fig{simplevae}), verifying that an estimator with limited capability can still help our method avoid mode collapse.

\begin{figure}[tb]
    \centering
    \begin{subfigure}[t]{0.4\textwidth}
        \includegraphics[width=\textwidth,height=0.7\textwidth]{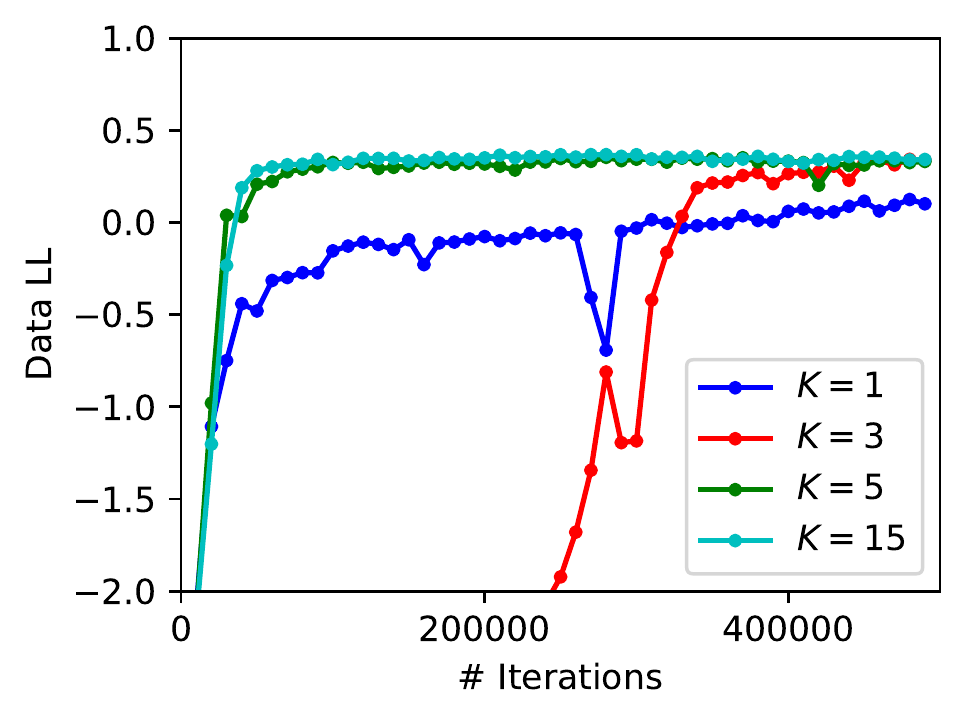}
    \end{subfigure}
    \begin{subfigure}[t]{0.4\textwidth}
        \includegraphics[width=\textwidth,height=0.7\textwidth]{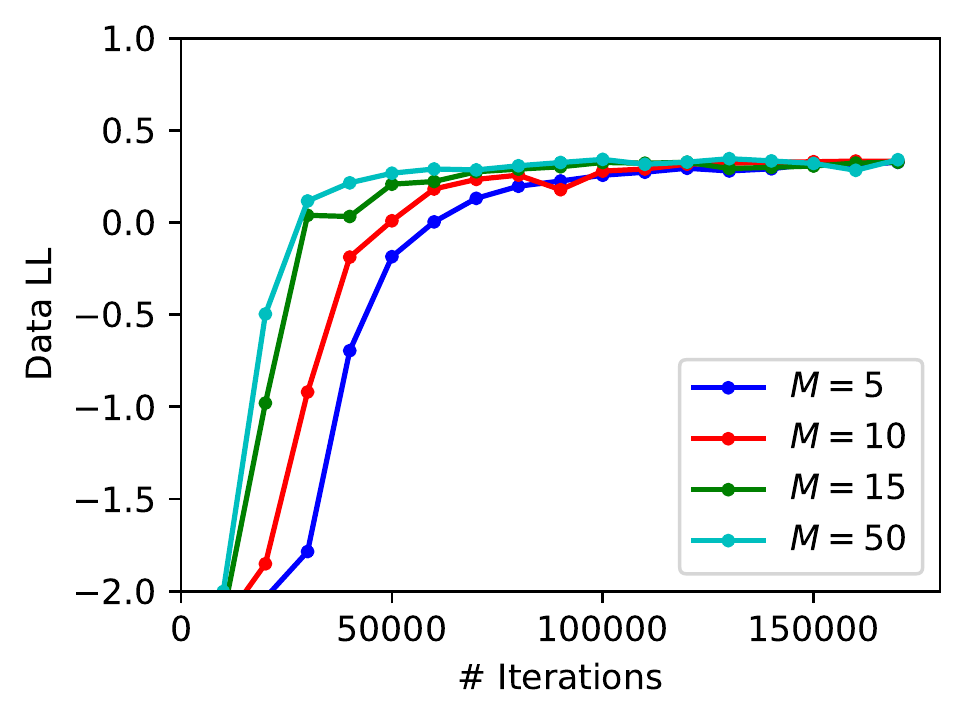}
    \end{subfigure}
	\caption{
Learning curves of LBT on the ring data with different unrolling steps $K$ and fixed $M=15$ (above) and with different estimator update steps $M$ and fixed $K=5$ (bottom).}\label{fig:sens}
\vspace{-.2cm}
\end{figure}

\vspace{-.1cm}
\subsection{CelebA \& CIFAR10}
\vspace{-.1cm}

We also evaluate LBT-GAN on natural images, including CIFAR10~\cite{krizhevsky2009learning} and CelebA~\cite{liu2015faceattributes} datasets. The generated samples of DCGANs and LBT-GANs are illustrated in \fig{color_image}.
LBT-GAN can generate images with comparable quality as DCGANs, demonstrating that LBT-GAN can successfully scale to natural images.
However, we observe LBT-GAN generates more diverse samples compared to DCGAN without careful fine-tuning, especially on the CelebA dataset.
The running time of LBT-GAN is roughly $2$ (or $3$) times of DCGAN on CelebA (or Cifar10). Empirically, LBT-GAN gives better or comparable results compared to VAE and VAE-GAN. Also, we find both GAN and VAE-GAN are sensitive to the architectures of G and D, whereas LBT-GAN is much more robust. This can be highlighted by adopting a relatively small G, where LBT-GAN with VAE can achieve 62.3 FID scores, significantly outperforming GAN(156.9), VAE(379.6) and VAE-GAN(100.5) on CelebA.

\vspace{-.1cm}
\subsection{Sensitivity Analysis of $K$ and $M$}\label{sec:sens}
\vspace{-.1cm}


Theoretically, a larger unrolling steps $K$ allows $\phi^K$ to better approximate $\phi^\star$ and a larger inner update iterations $M$ can better approximate the condition that $\phi^0=\phi^\star$ as analyzed in Appendix B.
However, large $K$ and $M$ on the other hand increase the computational costs.
To balance this trade-off, we provide sensitivity analysis of $K$ and $M$ in LBT. We use the experimental settings of the ring problem and adopt the values of the objective function \eqn{obj}, i.e., the log-likelihood of real samples evaluated by the learned estimator, as the quantitative measurement.

We first investigate the influence of the number of unrolling steps $K$ on the training procedure.
We vary the value of $K$ and show the learning curves with $K\in\{1,3,5,15\}$ in \fig{sens}. We observe that $K=1$ leads to a suboptimal solution and larger $K$ leads to better solution and convergence speed. We do not observe significant improvement with $K$ larger than $5$.
We also show the influence of the number of inner update iterations $M$ during training with  
$M\in\{5,10,15,50\}$ in \fig{sens}. Our observation is that larger $M$ leads to faster convergence, which is consistent with the analysis in \secref{method:if}.

\vspace{-.1cm}
\section{Conclusions \& Discussions}
\vspace{-.1cm}

We present a novel framework LBT to train an implicit generative model via teaching an auxiliary density estimator, which is formulated as a bilevel optimization problem. Unrolling techniques are adopted for practical optimization. Finally, LBT is justified both theoretically and empirically.

The main bottleneck of LBT is how to efficiently solve the bilevel optimization problem. For one thing, each update of LBT could be slower than that of the existing methods because the computational cost of the unrolling technique grows linearly with respect to the unrolling steps.
For another, LBT may need larger number of updates to converge than GAN because training a density estimator is more complicated than training a classifier. Overall, if the bilevel optimization problem can be solved efficiently in the future work, LBT can be scaled up to larger datasets.

LBT bridges the gap between the training of implicit models and explicit models. For one thing, the auxiliary explicit models can help implicit models overcome the mode collapse problems. For another, the implicit generators can be viewed as approximated samplers of the density estimators like auto-regressive models, from which getting samples is time-consuming. We discuss the former direction in this paper and leave the later direction as future work.

\vspace{-.1cm}
\section{Acknowledgements}
\vspace{-.1cm}

This work was supported by the National Key Research and Development Program of China (No. 2017YFA0700904), NSFC Projects (Nos. 61620106010, 61621136008, U19B2034, U181146), Beijing NSF Project (No. L172037), Beijing Academy of Artificial Intelligence (BAAI), Tsinghua-Huawei Joint Research Program, Tiangong Institute for Intelligent Computing, and the NVIDIA NVAIL Program with GPU/DGX Acceleration. C. Li was supported by the Chinese postdoctoral innovative talent support program and Shuimu Tsinghua Scholar.

 \bibliographystyle{splncs04}
\bibliography{bib}
\end{document}


\title{Learning Implicit Generative Models By Teaching Density Estimators}
%
%
\author{Anonymous Authors\inst{1}}

%
\institute{Anonymous Institute}
\maketitle
 
\appendix

\section{Proof of Theorem 1.}\label{sec:app:proof}

\begin{theorem}
	 Solving problem (Eqn. (2)) is equivalent to minimizing the KL-divergence between the data distribution and the generator distribution, and it's optima is achieved when
	 \begin{align}
	     p_G = p_E = p_{\data}. 
	 \end{align}
\end{theorem}

\begin{proof}
    For a fixed generator $G$, the objective of the estimator is to maximize the log-likelihood of the generated samples:
	\begin{align}
		 &\ep_{x\sim p_G(x;\theta)} [\log p_E(x;\phi)]  \label{eqn:proof} \\
        = &\ep_{x\sim p_G(x;\theta)} [\log \frac{ p_E(x;\phi)}{p_G(x;\theta)}] 
+ \ep_{x\sim p_G(x;\theta)} [\log p_G(x;\theta)] \nonumber\\
		=& -KL(p_G(x;\theta) \| p_E(x;\phi)) - H(p_G(x;\theta)),\nonumber
	\end{align}
    where $H(p_G(x;\theta))$ is the entropy of the generator distribution, which is a constant with respect to the estimator $E$. Hence, maximizing the above objective is equivalent to minimizing the KL-divergence between $p_G$ and $p_E$. The estimator $E$ thus achieves its optimum when $p_E=p_G$,
    i.e., $p_E(x;\phi^\star(\theta)) = p_G(x;\theta)$.
	Then it is straightforward that maximizing \eqn{proof} is equivalent to solving the problem:
	$\max_\theta \ep_{x\sim p_{\data}(x)}[\log p_G(x;\theta)],$ which 
	is equivalent to minimizing the KL-divergence between $p_{\data}(x)$ and $p_G(x;\theta)$,
	whose optimum is achieved when
    $p_G = p_{\data}.$
\end{proof}

\section{Connection to Influence Function}\label{sec:sup:if}


Here we build connections between the approximate gradients obtained by the unrolling technique and the exact gradients given by the influence function~\cite{koh2017understanding}. The influence function provides a way to calculate the gradient of $\phi^\star(\theta)$ w.r.t. the generated samples $x_z$ as follows:
\begin{align}\label{eqn:influ}
\frac{\partial \phi^\star(\theta)}{\partial x_z} = -H_{\phi^\star}^{-1}\nabla_{x_z} \left(\left.\frac{\partial f_E(\theta, \phi)}{\partial \phi}\right|_{\phi^\star}\right),
\end{align}
where $H_{\phi^\star}$ is the Hessian of the objective $f_E$ w.r.t. $\phi$ at $\phi^\star$ and is negative semi-definite~\cite{koh2017understanding}. 
However, calculating and inverting $H_{\phi^\star}$ is computationally expensive.

Assuming that the parameters of the density estimator is at its optimum $\phi^\star$, i.e., $\phi^0=\phi^\star$, we examine the case of $K=1$.
The result of one step unrolling is given by:
\begin{align}
&\frac{\partial \phi^1}{\partial x_z} = \frac{\partial (\phi^0 + \eta\nabla\phi)}{\partial x_z} = \eta\frac{\partial \nabla\phi}{\partial x_z} \label{eqn:unroll}\\
&= \eta\nabla_{x_z}\left(\left.\frac{\partial f_E(\theta, \phi)}{\partial \phi}\right|_{\phi^0}\right)
= \eta\nabla_{x_z}\left(\left.\frac{\partial f_E(\theta, \phi)}{\partial \phi}\right|_{\phi^\star}\right). \nonumber
\end{align}
Note that the inner product of $\frac{\partial \phi^1}{\partial x_z}$ and $\frac{\partial \phi^\star}{\partial x_z}$ given by the influence function is positive because the Hessian $H_{\phi^\star}$ is negative semi-definite~\cite{koh2017understanding}.
Therefore, the unrolling technique essentially gives an approximation of the influence function in \eqn{influ} under the condition that the estimator is good enough (near to its optimaum $\phi^\star$). Besides, the unrolling technique is much more efficient as it does not need to inverse the Hessian matrix.

\bibliographystyle{splncs04}
\bibliography{bib}